\documentclass[lettersize,journal]{IEEEtran}
\usepackage[utf8]{inputenc}

\usepackage{enumitem}

\usepackage{multibib}
\newcites{latex}{References for the Supplementary Materials}

\usepackage[colorlinks=true,
linkcolor=blue, 
citecolor=blue, 
urlcolor = black]{hyperref}

\usepackage{amsmath,amsfonts,bm}
\usepackage{algorithm}
\usepackage{algorithmicx}
\usepackage{algpseudocode}
\usepackage{amsthm}

\newtheorem{theorem}{Theorem}
\newtheorem{corollary}{Corollary}[theorem]

\makeatletter
\let\oldcitelatex\citelatex
\renewcommand{\citelatex}[1]{\textcolor{blue}{\oldcitelatex{#1}}}
\makeatother

\usepackage[T1]{fontenc}
\usepackage{xcolor}
\usepackage{subcaption}
\usepackage{microtype}

\usepackage{array}
\usepackage{textcomp}
\usepackage{stfloats}
\usepackage{url}
\usepackage{verbatim}
\usepackage{graphicx}
\usepackage{cite}

\def\eqref#1{equation~\ref{#1}}

\def\1{\bm{1}}

\def\vtheta{{\bm{\theta}}}

\def\vb{{\bm{b}}}
\def\vc{{\bm{c}}}

\def\vf{{\bm{f}}}
\def\vg{{\bm{g}}}
\def\vh{{\bm{h}}}

\def\vk{{\bm{k}}}

\def\vq{{\bm{q}}}

\def\vu{{\bm{u}}}
\def\vv{{\bm{v}}}

\def\vx{{\bm{x}}}
\def\vy{{\bm{y}}}
\def\vz{{\bm{z}}}

\def\mA{{\bm{A}}}
\def\mB{{\bm{B}}}
\def\mC{{\bm{C}}}

\def\mG{{\bm{G}}}

\def\mI{{\bm{I}}}

\def\mK{{\bm{K}}}

\def\mQ{{\bm{Q}}}

\def\mV{{\bm{V}}}
\def\mW{{\bm{W}}}
\def\mX{{\bm{X}}}
\def\mY{{\bm{Y}}}
\def\mZ{{\bm{Z}}}

\def\mLambda{{\bm{\Lambda}}}

\DeclareMathAlphabet{\mathsfit}{\encodingdefault}{\sfdefault}{m}{sl}
\SetMathAlphabet{\mathsfit}{bold}{\encodingdefault}{\sfdefault}{bx}{n}

\def\sA{{\mathbb{A}}}

\def\sF{{\mathbb{F}}}
\def\sG{{\mathbb{G}}}

\def\sM{{\mathbb{M}}}

\newcommand{\R}{\mathbb{R}}

\begin{document}
\title{Physics-Informed Neural Koopman Machine for Interpretable Longitudinal Personalized Alzheimer's Disease Forecasting}

\author{Georgi Hrusanov\thanks{G. Hrusanov is with the Platform of Bioinformatics (PBI), University Lausanne Hospital (CHUV) and Faculty of Biology and Medicine (FBM), University of Lausanne (UNIL), Lausanne, 1012, Switzerland. (email: georgi.hrusanov@chuv.ch).}, Duy-Thanh Vu, Duy-Cat Can\thanks{D.T. Vu, D.C. Can, M. Ryan, and J. Bodelet are with PBI, CHUV and FBM, UNIL, Lausanne, 1012, Switzerland. S. Tascedda is with the Department of Biomedical Engineering, University of Basel, 4123 Basel, Switzerland.}, Sophie Tascedda, Margaret Ryan, Julien Bodelet,\\
Katarzyna Koscielska$^\ast$, Carsten Magnus$^\ast$\thanks{K. Koscielska is with Roche Diagnostics International Ltd., Rotkreuz, 6343, Switzerland. C. Magnus is with F. Hoffmann - La Roche Ltd., Roche Information Solutions, Basel, 4058, Switzerland.}, and Oliver Y. Ch\'{e}n$^\ast$\thanks{O.Y. Ch\'{e}n is with the PBI, CHUV and FBM, UNIL, Lausanne, 1012, Switzerland. (email: olivery.chen@chuv.ch).}
\thanks{$^\ast$K. Koscielska, C. Magnus, and O.Y. Ch\'{e}n are joint senior authors.}
\\
for the Alzheimer's Disease Neuroimaging Initiative$^\dagger$\thanks{$^\dagger$Data used in preparation of this article were obtained from the Alzheimer's Disease Neuroimaging Initiative (ADNI) database (adni.loni.usc.edu). As such, the investigators within the ADNI contributed to the design
and implementation of ADNI and/or provided data but did not participate in analysis or writing of this report.
A complete listing of ADNI investigators can be found at:
\protect\url{http://adni.loni.usc.edu/wp-content/uploads/how_to_apply/ADNI_Acknowledgement_List.pdf }.}
\thanks{}
\thanks{}}


\maketitle

\begin{abstract}
Early forecasting of individual cognitive decline in Alzheimer's disease (AD) is central to disease evaluation and management. Despite advances, it is as of yet challenging for existing methodological frameworks to integrate multimodal data for longitudinal personalized forecasting while maintaining interpretability. To address this gap, we present the Neural Koopman Machine (NKM), a new machine learning architecture inspired by dynamical systems and attention mechanisms, designed to forecast multiple cognitive scores simultaneously using multimodal genetic, neuroimaging, proteomic, and demographic data. NKM integrates analytical ($\alpha$) and biological ($\beta$) knowledge to guide feature grouping and control the hierarchical attention mechanisms to extract relevant patterns. By implementing Fusion Group-Aware Hierarchical Attention within the Koopman operator framework, NKM transforms complex nonlinear trajectories into interpretable linear representations. To demonstrate NKM's efficacy, we applied it to study the Alzheimer's Disease Neuroimaging Initiative (ADNI) dataset. Our results suggest that NKM consistently outperforms both traditional machine learning methods and deep learning models in forecasting trajectories of cognitive decline. Specifically, NKM (1) forecasts changes of multiple cognitive scores simultaneously, (2) quantifies differential biomarker contributions to predicting distinctive cognitive scores, and (3) identifies brain regions most predictive of cognitive deterioration. Together, NKM advances personalized, interpretable forecasting of future cognitive decline in AD using past multimodal data through an explainable, explicit system and reveals potential multimodal biological underpinnings of AD progression.
\end{abstract}


\section{Introduction}
Alzheimer's disease (AD) is a growing global crisis. Recent estimates from the Alzheimer's Disease International project that AD will affect up to 139 million cases worldwide by 2050 
\cite{ADI_Dementia_Statistics_2023}
with an economic burden of nearly \$17 trillion \cite{international_world_2023, nandi_global_2022}. AD progressively degrades patients' cognitive, behavioral, and social faculties, severely diminishing quality of life and inevitably leading to death \cite{kahle-wrobleski_assessing_2017}.

Early forecasting of AD progression is critical to disease management and to developing targeted treatments, but it must confront several challenges. First, because of AD's neural, genetic, proteomic, and demographic associations, AD forecasting is a multimodal problem that requires investigating both inter- and intra-modality relationships. Current machine learning (ML) methods, however, often overlook these relationships \cite{li_deep_2014,wen_convolutional_2020}. Second, AD disrupts multiple cognitive functions. Longitudinal prediction of multivariate outcomes, therefore, is crucial \cite{morris_clinical_1997}. Third, AD progression is complex, with nonlinear and highly heterogeneous trajectories across patients \cite{jack_nia-aa_2018, livingston_dementia_2020, vogel_four_2021, scholl_distinct_2017, diaz_optimus_2025}. Fourth, to establish clinical utility, AD prediction needs to be interpretable, ranking the hierarchy of feature importance and mapping neuroimaging biomarkers back onto the brain space.

Chief to addressing these challenges is solving the problem of \textit{longitudinal many-to-many forecasting}: using past multimodal data to predict future multivariate outcomes, while maintaining biological interpretability. Formally, consider a multimodal feature vector measured at time $t$, $\vx_t$, as $\vx_t = [\vx_t^{(1)}, \vx_t^{(2)}, \vx_t^{(3)}, \vx_t^{(4)}, \vx_t^{(5)}]$,
where the multimodal data include genetic markers ($\vx_t^{(1)}$, APOE4 genotype); proteomic biomarkers ($\vx_t^{(2)}$, including amyloid beta, total tau or t-tau, and phosphorylated tau or p-tau from the cerebrospinal fluid (CSF)); PET imaging biomarkers ($\vx_t^{(3)}$, including FDG, PIB, and AV-45 uptake); and cortical thickness features and intracranial volume (ICV) from the MRI data ($\vx_t^{(4)}$), plus the optional non-biomarker demographic factors such as age, sex, and education ($\vx_t^{(5)}$), yielding a total of 44 features per time step to capture both biological and clinical covariates essential for more in-depth personalized modeling.

Denote multivariate clinical outcomes $\vy_{t+1} \in \R^c$ for $c$ disease outcomes evaluated at time $t+1$. The central aim of the problem is to find a biologically interpretable mapping by leveraging Koopman operator theory to linearize nonlinear disease dynamics, such that:
\begin{equation}
    \hat{\mathbf{y}}_{t+1} = \mathbf{D}\!\left( \mathbf{K} \mathbf{z}_t^{\text{ref}} + \mathbf{c}_t \right) ,
\end{equation}
where $\mathbf{D}$ is the decoder, $\mathbf{K}$ is the learned Koopman matrix, $\vz_t^{\text{ref}}$ is the refined latent representation derived from the temporal window $\{\vx_\tau : \tau \in \{t-w+1,\dots,t\}\}$, and $\vc_t$ is the hierarchical attention-derived control vector.

To address the longitudinal many-to-many forecasting problem, we propose the Neural Koopman Machine (NKM), a new, physics-inspired architecture combining two complementary knowledge types: \textbf{biological knowledge} ($\beta$-knowledge) and \textbf{analytical knowledge} ($\alpha$-knowledge). Specifically, $\beta$-knowledge encodes domain-specific information on multimodal biomarker organization (genetic, CSF, PET, MRI, and demographics) to guide feature partitioning and processing. The $\alpha$-knowledge implements the mathematical framework for Koopman operator linearization, matrix operations, and hierarchical attention mechanisms, which dynamically weigh biomarker relevance across time and modalities. By bridging physics-informed modeling with domain expertise, NKM not only advances predictive accuracy but also lays the foundation for precision medicine in neurodegeneration, enabling early intervention, patient stratification in clinical trials, and potentially the discovery of new therapeutic targets.

NKM's contributions are threefold: methodology, neurobiology, and explainability. From a methodological perspective, NKM leverages Koopman operator theory~\cite{koopman_hamiltonian_1931, brunton_modern_2022} to linearize nonlinear dynamics, improving on previous disease models ~\cite{young_data-driven_2014, oxtoby_data-driven_2017} by integrating multimodal data via structured, interpretable attention mechanisms. Biologically, unlike prior applications of Koopman theory in healthcare~\cite{proctor_generalizing_2018}, NKM directly incorporates biological insights into the analytical framework. Regarding explainability, unlike black-box deep learning approaches~\cite{bhagwat_modeling_2018, lee_predicting_2019}, NKM explicitly quantifies temporal and biomarker contributions to each personalized forecast by leveraging physics-informed techniques.

\begin{figure*}[ht]
    \centering
    \includegraphics[width=0.95\textwidth]{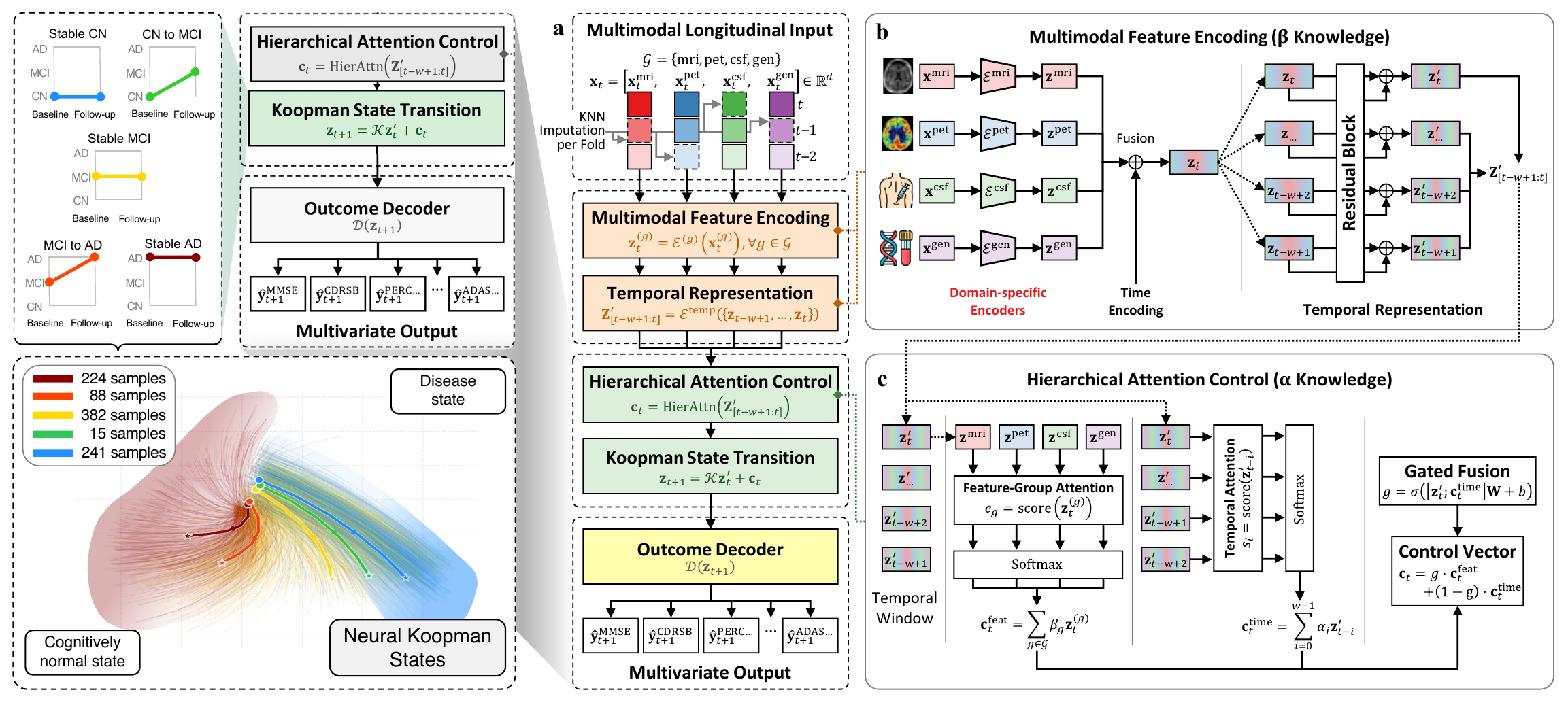}
    \caption{\textbf{The schematic architecture of the Neural Koopman Machine (NKM).} \textbf{Right panel (a-c):} \textcolor{black}{NKM uses biological knowledge ($\beta$-knowledge) to guide modality-specific encoders to process} Magnetic Resonance Imaging (MRI), Positron Emission Tomography (PET), cerebrospinal fluid (CSF), genetic, and demographic data with time embeddings, \textcolor{black}{and to fuse the resulting learned modality-specific latent representations} into a \textcolor{black}{joint} latent state. In parallel, NKM uses analytical knowledge ($\alpha$-knowledge) to steer hierarchical attention to combine feature-group and temporal attention and generate a control vector, which drives a Koopman state transition. The decoder then maps the predicted latent state to multiple clinical outcome trajectories (e.g., MMSE, CDR-SB, and ADAS). \textbf{Left panel (inset):} \textcolor{black}{The learned} Koopman latent-space trajectories highlight distinct dynamical patterns for subjects who are stably cognitively normal (CN $\rightarrow$ CN), with mild cognitive impairment (MCI $\rightarrow$ MCI) and Alzheimer's disease (AD $\rightarrow$ AD), as well as converters (e.g., CN $\rightarrow$ MCI and MCI $\rightarrow$ AD).}

    \label{fig:overall_architecture}
\end{figure*}

\section{Related Work}
\label{sec:related}
\subsection{Dynamical Systems Modeling}
Traditional dynamical systems such as Ordinary Differential Equations (ODEs) and state-space models offer interpretability but often fail to capture the nonlinearities of biomedical data \cite{murray_mathematical_2002}. More flexible approaches, such as Hidden Markov Models (HMMs), Gaussian Processes (GPs), and Kalman filters, struggle with scalability and high-dimensional feature interactions \cite{ghahramani_tutorial_nodate,kalman_new_1960}. Recent developments, such as Neural and Latent ODEs, combine deep learning with differential equations \cite{chen_neural_2019, rubanova_latent_2019}, but computational efficiency and explicit interpretability remain major obstacles to their adoption in clinical practice.

\subsection{Koopman Operator Theory (KOT) in Machine Learning}
KOT enables the linear representation of nonlinear dynamical systems via infinite-dimensional linear operators on state observables \cite{koopman_hamiltonian_1931,mezic_spectral_2005}. Data-driven methods based on KOT, such as Dynamic Mode Decomposition (DMD) and Extended Dynamic Mode Decomposition (EDMD), are useful to analyse nonlinear dynamical systems\cite{schmid_dynamic_2010,williams_data-driven_2015}. In parallel, neural networks help learn Koopman latent representations directly from data \cite{lusch_deep_2018,yeung_learning_2017}. However, despite their theoretical and methodological advances, as well as applications in studying fluid dynamics, there has been comparatively less emphasis on biomedical time-series and disease progression modeling.

\subsection{Deep Learning for Disease Progression}
A powerful way to model complex temporal dynamics in longitudinal clinical data is to use deep learning. Recurrent Neural Networks (RNNs), such as Long Short-Term Memory (LSTM) and Gated Recurrent Units (GRUs), show promise for neurodegenerative disease prediction \cite{mehdipour_ghazi_training_2019,lee_predicting_2019}. Models such as time-aware LSTMs and GRU-D are useful for handling irregular sampling, as they incorporate temporal awareness or hierarchical attention \cite{baytas_patient_2017, che2018recurrent, choi_retain_2016}. Additionally, deep learning approaches combining neural networks and survival analysis show promise to manage long-range temporal dependencies \cite{rajkomar_scalable_2018,vaswani_attention_2017-1}. Linking deep learning with dynamical systems, recent work, such as neural ODEs \cite{chen_neural_2019}, continuous-time modeling methods \cite{rubanova_latent_2019}, and physics-informed NNs \cite{raissi_physics-informed_2019}, is beginning to unveil dynamic disease evolution. Finally, while deep learning methods, such as transformer- and convolution-based approaches, have shown notable performance in disease prediction \cite{kitaev2020reformerefficienttransformer, bai_empirical_2018}, most deep learning models at present lack interpretability.

\subsection{Attention Mechanisms in Biomedical Applications}
Through selective weighting, attention mechanisms advance sequence modeling   \cite{vaswani_attention_2017-1}. In biomedical applications, such self-attention helps reveal clinically significant patterns \cite{liu_understanding_2023}, handle irregular time series \cite{luo_hitanet_2020}, and, with topological priors, capture patient trajectory structure \cite{li_behrt_2020}. For complex biomedical data analysis, hierarchical attention \cite{ma_dipole_2017} is particularly promising, as it can selectively focus on relevant features both within and across modalities and over time.

\section{Physics-informed Neural Koopman Machine}
\label{sec:method}

\subsection{Notations}
Here, we present the notations used throughout this paper. We denote $\vx$, $\vz$, $\vc$, and $\vy$ as the features, their latent representations, control signals, and disease-related outcomes, respectively. We let $\mathcal{K}$ and $\mathbf{K}$ denote the infinite-dimensional Koopman operator and its finite-dimensional approximation, respectively. We use $\boldsymbol{\Phi}$ and $\phi$ to represent the Koopman observables in their vector and scalar forms, respectively. We use Greek letters, $\mW$, and $\vb$ to denote model parameters. We utilize $f$, $g$, and $\sigma$ to represent smooth functions. We use superscripts $^{\text{ref}}$, $^{\text{enc}}$, $^{\text{feat}}$, $^{\text{time}}$ to indicate \textit{refined} and \textit{encoded} representations, and \textit{feature}- and \textit{time}-related signals, respectively. We use subscripts $_{\text{enc}}$, $_{\text{dec}}$, $_{\text{pred}}$, $_{\text{koop}}$, and $_{\text{spectral}}$ to indicate parameters related to the \textit{encoder} and \textit{decoder}, and losses related to \textit{prediction}, the \textit{Koopman} term, and its \textit{spectral} regularization.

\subsection{Data}

Data used in the preparation of this article were obtained from the Alzheimer's Disease Neuroimaging
Initiative (ADNI) database \url{(adni.loni.usc.edu)}. The ADNI was launched in 2003 as a public-private
partnership, led by Principal Investigator Michael W. Weiner, MD. The primary goal of ADNI has been to
test whether serial magnetic resonance imaging (MRI), positron emission tomography (PET), other
biological markers, and clinical and neuropsychological assessment can be combined to measure the
progression of mild cognitive impairment (MCI) and early Alzheimer's disease (AD).

\subsection{Preliminaries of the Koopman Operator Theory (KOT)}

KOT enables the linear representation of nonlinear systems by lifting dynamics to a higher-dimensional space. For a nonlinear system $\vx_{t+1} = f(\vx_t)$, the Koopman operator $\mathcal{K}$ acts on observable functions: $(\mathcal{K}g)(\vx) = g(f(\vx))$. Specifically, we approximate this using $d_z$ observables $\boldsymbol{\Phi}(\vx) = [\phi_1(\vx), \ldots, \phi_{d_z}(\vx)]^\top$ to obtain a lifted representation $\vz_t = \boldsymbol{\Phi}(\vx_t)$, where dynamics become linear:
\begin{equation}
\label{alg:lifting}
\vz_{t+1} \approx \mathbf{K} \vz_t,
\end{equation}
with $\mathbf{K} \in \R^{d_z \times d_z}$ approximating the infinite-dimensional operator $\mathcal{K}$, where $d_z$ is the latent dimension. \textbf{Eq.} \textcolor{blue}{\ref{alg:lifting}}, therefore, enables linear analysis of complex nonlinear disease dynamics (in autonomous cases; we extend to controlled dynamics below).

\subsection{Our Contributions to Improving the KOT}

We formulate the longitudinal many-to-many forecasting problem as follows: one seeks to predict future disease-related cognitive scores $\vy_{t+1} \in \R^3$ (e.g., MMSE, CDRSB, and ADAS13) at time $t+1$ using observed multimodal biomarkers $\vx_t \in \R^{44}$ at three time points up to time $t$, where the multimodal biomarkers consist of 44 features, including one genetic feature (APOE4), three CSF features (Amyloid-beta (A$\beta$), t-tau, and p-tau protein), three PET imaging features (FDG, PIB, and AV45), 18 MRI imaging features (17 Schaefer networks assembled from 202 cortical thickness regions and Intracranial Volume (ICV)), and demographic information (two numeric and 17 categorical features).

Our work introduces three contributions to advance KOT for future disease forecasting using past multimodal data.

\textbf{Contribution 1: Simultaneous Learning of Observables and Koopman Dynamics.} Unlike classical Extended Dynamic Mode Decomposition (EDMD), which requires predefined observables, we jointly learn both the lifting function $\boldsymbol{\Phi}(\cdot)$ and the Koopman matrix $\mathbf{K}$ in an end-to-end fashion. Specifically, we parameterize:
\begin{equation}
\vz_t = \boldsymbol{\Phi}(\vx_t; \vtheta_{\text{enc}}), 
\qquad 
\vz_{t+1} = \mathbf{K} \vz_t^{\text{ref}} + \vc_t,
\end{equation}
where $\boldsymbol{\Phi}(\cdot; \vtheta_{\text{enc}})$, $\mathbf{K} \in \R^{d_z \times d_z}$, $\vz_t^{\text{ref}}$ (refined latent representation via residual blocks) and $\vc_t$ are a neural encoder with parameters $\vtheta_{\text{enc}}$, the learnable Koopman matrix, the temporally refined latent representation, and an attention-derived control vector, respectively (see \textbf{Sec.} \textcolor{purple}{\ref{subsec:architecture}}).

\textbf{Contribution 2: Biologically-Informed Feature Architecture.} We decompose the input feature space into biologically meaningful modalities $\vx_t = [\vx_t^{(g)} : g \in \sG]$, where $\sG = \{\text{genetic, CSF, PET, MRI, demographics}\}$. Each modality is processed by dedicated, modality-specific encoders before fusion. This approach preserves modality-specific biological structures and enables cross-modal interactions in the lifted space; $\beta$-knowledge guides the aggregation (e.g., Schaefer regions to 17 networks).

\textbf{Contribution 3: Hierarchical Attention-Based \mbox{Control}.}
We extend the autonomous Koopman dynamics to incorporate
personalized control via hierarchical attention that processes
temporal windows of refined latent states $[\vz_{t-w+1}^{\text{ref}}, \ldots, \vz_t^{\text{ref}}]$, 
where $w$ is the window size. The control vector $\vc_t$ adapts to
relevant biomarker patterns and subject-specific longitudinal trajectories 
across both temporal and modality dimensions.

\subsection{Optimizing the Multimodal Longitudinal Forecasting Problem: Technical Overview}

The complete pipeline connects input biomarkers to cognitive predictions via:
$\vx_t \xrightarrow{\boldsymbol{\Phi}} \vz_t \xrightarrow{\text{refine}} \vz_t^{\text{ref}}
\xrightarrow{\mathbf{K}, \vc_t} \vz_{t+1} \xrightarrow{\mathcal{D}} \hat{\vy}_{t+1}$,
where $\mathcal{D}:= \mathcal{D}(\cdot; \vtheta_{\text{dec}})$ is a neural decoder with parameters
$\vtheta_{\text{dec}}$. Our joint optimization objective combines two complementary loss terms:
\begin{subequations}\label{eq:total_loss_short}
\begin{align}
\mathcal{L} &= \mathcal{L}_{\text{pred}} + \lambda\,\mathcal{L}_{\text{koop}} + \mathcal{R}_{\text{spec}}, \label{eq:total_loss_main}\\
\mathcal{L}_{\text{pred}} &= \frac{1}{N}\sum_{i=1}^N \|\mathcal{D}(\vz_{t+1}^{\text{ref},(i)}; \vtheta_{\text{dec}}) \nonumber \\
&\quad - \vy_{t+1}^{(i)}\|_2^2, \label{eq:Lpred}\\
\mathcal{L}_{\text{koop}} &= \frac{1}{N(T-1)}\sum_{i=1}^N\sum_{t=1}^{T-1} \|\vz_{t+1}^{\text{ref},(i)} \nonumber \\
&\quad - \mathbf{K}\vz_t^{\text{ref},(i)} - \vc_{\text{control}}^{(i)}\|_2^2, \label{eq:Lkoop}\\
\mathcal{R}_{\text{spec}} &= \eta \big( \max\{0, \|\mathbf{K}\|_2^2 - \rho^2\} \big)^2. \label{eq:Lspec}
\end{align}
\end{subequations}

Here, $\mathcal{L}_{\text{pred}}$ enforces accurate prediction of cognitive scores across the $N$ sequences, $\mathcal{L}_{\text{koop}}$ ensures Koopman consistency by requiring that the predicted latent evolution matches the true encoded next state, and \textcolor{black}{$\mathcal{R}_{\text{spec}}$ regularizes the Koopman matrix to ensure stable latent dynamics over time.}

\subsection{The $\alpha$ Knowledge and $\beta$ Knowledge}

\textit{The analytical ($\alpha$) knowledge} is the \textit{analytical framework} of the NKM. It leverages the Koopman operator's linearization capabilities with hierarchical attention and temporal encoding to convert nonlinear disease dynamics into linear latent-space representations. \textit{The biological ($\beta$) knowledge} is the \textit{biological insights} which organize the multimodal feature space — including functional brain network features (aggregated via Schaefer networks), imaging, CSF biomarkers, genetic markers, and demographics — into biologically relevant categories. By leveraging known intra- and inter-modality relationships, $\beta$ knowledge guides a feature-aware encoding process instead of treating all inputs equally. 

Below, we detail how these knowledge types are implemented as specific architectural components.

\subsection{The Physics-informed Neural Koopman Machine: Architecture Overview}
\label{subsec:architecture}
NKM is a physics-informed, modular machine learning architecture that integrates the Koopman operator theory (KOT), hierarchical attention (via $\alpha$ knowledge), and biologically structured encoding (via $\beta$ knowledge) to simultaneously predict
future changes in multiple AD-related cognitive scores based on past multimodal data.

In brief, the NKM maps nonlinear multimodal biomarker trajectories into Koopman space, where dynamics evolve linearly under personalized control.
NKM consists of five sequential components:
    Multimodal Feature Encoding,
    Temporal Representation,
    Hierarchical Attention Control,
    Koopman State Transition,
    and Outcome Decoder (see \textbf{Fig.} \textcolor{blue}{\ref{fig:overall_architecture}} for the architecture of the NKM and Algorithm~\ref{alg:nkm} for a summary of its training pipeline).

We detail each component in the following subsections.

\begin{algorithm}[t]
\caption{Neural Koopman Machine Training Pipeline}
\label{alg:nkm}
\small
\begin{algorithmic}[1]
\Require Modality groups $\sG$, multimodal sequences $\{\vx_t^{(g)}\}_{g\in\sG}$, cognitive scores $\{\vy_t\}$, window size $w=3$
\State Initialize encoders $\{\mathcal{E}^{(g)}\}$, Koopman matrix $\mathbf{K}$, attention modules, decoder $\mathcal{D}$
\For{each training batch}
  \ForAll{$g \in \sG$}
    \State $\vz_t^{(g)} \gets \mathcal{E}^{(g)}\!\big(\vx_t^{(g)}\big)$
  \EndFor
  \State $\vz_t^{\text{enc}} \gets \operatorname{Fusion}\!\left([\vz_t^{(g)}]_{g \in \sG}\right)$
  \For{$t' \gets t-w+1$ \textbf{to} $t$}
    \State $\vz_{t'}^{\text{ref}} \gets \vz_{t'}^{\text{enc}} + \operatorname{ResBlock}\!\big(\vz_{t'}^{\text{enc}}\big)$
  \EndFor
  \Statex
  \State $\vc_t \gets \operatorname{HierAttn}\!\left([\vz_{\tau}^{\text{ref}}]_{\tau=t-w+1}^{t}\right)$
  \State $\vz_{t+1} \gets \mathbf{K}\,\vz_t^{\text{ref}} + \vc_t$
  \State $\hat{\vy}_{t+1} \gets \mathcal{D}(\vz_{t+1})$
  \State Compute loss $\mathcal{L}$ and update parameters
\EndFor
\end{algorithmic}
\end{algorithm}

\subsection{Multimodal Feature Encoding Using $\beta$ knowledge}

The $\beta$ knowledge guides the NKM to first partition features based on modalities and then process them using a modality-specific encoder, thereby preserving biological interpretability, maintaining meaningful representations, and enabling group-level attention reasoning in later stages. At each timestep $t$, NKM decomposes the input $\vx_t \in \R^{25}$ (or $\vx_t \in \R^{44}$ when including demographic information) into feature groups: genetics, CSF, PET, functional brain networks (Schaefer's 17 networks), and, when available, demographics.
A modality-specific encoder then processes each modality $g \in \sG$ as follows:
\begin{equation}
\vz_t^{(g)} = \mathcal{E}^{(g)} \big( \vx_t^{(g)} \big),
\quad
\vz_t^{\text{enc}} = \texttt{Fusion}\!\left(\left[\vz_t^{(g)} : g \in \sG\right]\right).
\end{equation}
For each modality $g$, $\vz_t^{(g)} \in \R^{d_g}$, and the fusion operation (with projection) yields $\vz_t^{\text{enc}} \in \R^{d_z}$. These encoders learn data-driven approximations of the Koopman observables $\boldsymbol{\Phi}(\vx_t)$. Here, we consider the core biological features as $\R^{25}$ (including 17 Schaefer's networks, three PET features, three CSF features, one APOE4 genetic feature, and one intracranial volume (ICV) feature);  the full input vector, including demographic information, is $\R^{44}$ for comprehensive modeling.

\subsection{Temporal Representation}
\label{subsec:temporal-attention-control}

Subsequently, the NKM refines the encoded latent vectors over a temporal window $\{\vz_{t-w+1}^{\text{enc}}, \dots, \vz_t^{\text{enc}}\}$: $\vz_{\tau}^{\text{ref}} = \vz_{\tau}^{\text{enc}} + \texttt{ResBlock}(\vz_{\tau}^{\text{enc}}), 
\quad \forall \tau \in \{t-w+1,\dots,t\}$.
This residual transformation, consisting of five blocks, enhances temporal structure while preserving the original signals, thereby improving the modeling of disease trajectories.

\subsection{Hierarchical Attention Control Using $\alpha$ knowledge}
\label{subsec:attention-control}
To adaptively guide latent evolution based on historical multimodal biomarkers, NKM employs a hierarchical attention mechanism that encodes both time- and modality-specific context. Formally, given the refined latent window $[\vz_{t-w+1}^{\text{ref}}, \dots, \vz_t^{\text{ref}}]$ and modality-specific embeddings $\{\vz_t^{(g)}\}_{g \in \sG}$, the attention module computes a pair of weighted summaries:
\begin{equation}
    \vc_t^{\text{time}} = \sum_{i=0}^{w-1} \alpha_i \, \vz_{t-i}^{\text{ref}}
    \qquad\text{and}\qquad
    \vc_t^{\text{feat}} = \sum_{g \in \sG} \beta_g \, \vz_t^{(g)},
\end{equation}
where
$\alpha_i = \frac{\exp\!\big(\frac{\vq_t^\top \vk_{t-i}}{\sqrt{d_k}}\big)}{\sum_{j=0}^{w-1} \exp\!\big(\frac{\vq_t^\top \vk_{t-j}}{\sqrt{d_k}}\big)}$ 
and 
$\beta_g = \frac{\exp\!\big(\frac{\vv_t^\top \vu_g}{\sqrt{d_k}}\big)}{\sum_{g' \in \sG} \exp\!\big(\frac{\vv_t^\top \vu_{g'}}{\sqrt{d_k}}\big)}$
are attention weights computed using scaled dot-product attention, with query $\vq_t$, keys $\vk_\tau = \mW_k \vz_\tau^{\text{ref}}$ (similarly for $\vu_g$), and scaling factor $\sqrt{d_k}$ for temporal attention, and feature query $\vv_t$ and keys $\vu_g$ for modality attention. This formulation weights recent visits more heavily while allowing data-driven weighting of both temporal and modality-specific information.

A learned gating function then fuses time- and feature-specific control signals, namely $\vc_t^{\text{time}}$ and $\vc_t^{\text{feat}}$, into a unified control vector $\vc_t$:
\begin{equation}
\vc_t = g \cdot \vc_t^{\text{feat}} + (1 - g) \cdot \vc_t^{\text{time}}, 
\qquad 
g = \sigma\!\big([\vz_t^{\text{ref}}; \vc_t^{\text{time}}]\mW_g + \vb_g\big),
\end{equation}
where $\sigma(\cdot)$ is the sigmoid function, $\mW_g \in \R^{(d_z + d_z) \times d_z}$, and $\vb_g \in \R^{d_z}$. The concatenation allows the gating mechanism to consider both the current refined state and the temporal summary, ensuring a balanced fusion. This attention-derived control signal $\vc_t$ adaptively steers the Koopman state transition, enabling personalized modeling of the disease-related cognitive score dynamics and highlighting critical visits and key biomarkers for each individual.

\subsection{Koopman State Transition}
To evolve the latent state, we use a learned Koopman matrix $\mathbf{K}$ with attention-based control:
\begin{equation}
\vz_{t+1} = \mathbf{K} \vz_t^{\text{ref}} + \vc_t,
\end{equation}
where $\vz_t^{}$ denotes the refined latent representation after temporal processing. This formulation decomposes disease dynamics into two complementary components: the Koopman matrix $\mathbf{K}$ models short-term temporal dependencies and variations around the mean trajectory, while the control vector $\mathbf{c}_t$ encodes long-term disease progression and personalized deviations specific to each patient's clinical profile. The spectral constraint $\|\mathbf{K}\|_2 < 1$ ensures stability of the temporal component without limiting the model's ability to capture progressive cognitive decline through $\mathbf{c}_t$.

The matrix $\mathbf{K} \in \R^{d_z \times d_z}$ approximates the infinite-dimensional Koopman operator $\mathcal{K}$, modeling global cognitive dynamics, while $\vc_t$ captures personalized deviations via the attention-derived control vector. This formulation generalizes the autonomous relation $\vz_{t+1} = \mathbf{K} \vz_t$ by incorporating patient-specific dynamics. Unlike classical EDMD, we jointly learn both the observables $\boldsymbol{\Phi}$ (via neural encoders) and the Koopman matrix $\mathbf{K}$ in an end-to-end manner, enabling adaptation to cognitive decline trajectories.

\subsection{Forecasting Future Disease Outcome}
Finally, we predict multiple future cognitive scores from the evolved latent state using a simple regression head:
\begin{equation}
    \hat{\vy}_{t+1} = \mathcal{D}(\vz_{t+1}).
\end{equation}
The decoder simultaneously predicts MMSE, CDRSB, and ADAS13, allowing the NKM to capture shared structure across cognitive landscapes without having to train three different neural networks.

\subsection{Training and Optimization}

We train the NKM end-to-end using the composite loss in \textbf{Eq.} \textcolor{blue}{\ref{eq:total_loss_short}}, augmented with spectral regularization to ensure stable long-term predictions:
\begin{equation}
\label{eq:spectral_reg}
\mathcal{R}_{\text{spectral}}
= \bigl(\max\{0,\, \|\mathbf{K}\|_2^2 - \rho^2\}\bigr)^2 ,
\end{equation}
where $\rho \in (0,1)$ is a target upper bound for the spectral norm of $\mathbf{K}$. We optimize the objective using AdamW with early stopping based on validation loss, and select the hyperparameters $\lambda$ and $\gamma$ via cross-validation.

\underline{Stability and accuracy of Koopman predictions}. The above regularization ensures that $\|\mathbf{K}\|_2 < 1$, providing a theoretical guarantee of stable out-of-sample predictions. The three loss components act complementarily: $\mathcal{L}_{\text{pred}}$ enforces accurate cognitive score prediction in the output space, $\mathcal{L}_{\text{koop}}$ ensures linear evolution in the latent space to maintain Koopman consistency, and $\mathcal{R}_{\text{spectral}}$ constrains the dynamics to remain stable over long time horizons. See the error bound analysis in \textbf{Sec.} \textcolor{purple}{\ref{subsec:theory}} for details.

\subsection{Theoretical Properties}
\label{subsec:theory}
\begin{theorem}[Error bound for $\varepsilon$-approximately Koopman-invariant observables]
\label{thm:error_bound}
If the learned observables $\boldsymbol{\Phi}$ are $\varepsilon$-approximately Koopman invariant, i.e.,
\begin{equation}
  \Bigl\|\mathcal{K}\phi_j - \sum_i K_{ij}\phi_i \Bigr\|_{\mathcal{L}^2(\sA)} \le \varepsilon ,
\end{equation}
then, for any $\tau \ge 1$,
\begin{align}
\bigl\|\boldsymbol{\Phi}(\vx_{t+\tau}) - \mathbf{K}^\tau \boldsymbol{\Phi}(\vx_t)\bigr\|_2
&\le \varepsilon \sqrt{d_z}\sum_{i=0}^{\tau-1}\|\mathbf{K}\|_2^i \label{eq:geom-series}\\
&= \frac{\varepsilon \sqrt{d_z}\bigl(1-\|\mathbf{K}\|_2^{\,\tau}\bigr)}{1-\|\mathbf{K}\|_2},
\quad \|\mathbf{K}\|_2<1. \label{eq:closed-form}
\end{align}
\end{theorem}

\begin{proof}  
See \textbf{Sec.} \ref{Proof:thm1 and corollary1}.
\end{proof}

\begin{corollary}[Asymptotic error bound]
\label{thm:asymptotic_error_bound}
It follows that, for long-term predictions (as $\tau\to\infty$): 
\begin{equation}
\lim_{\tau \to \infty}
\bigl\|\boldsymbol{\Phi}(\vx_{t+\tau}) - \mathbf{K}^\tau \boldsymbol{\Phi}(\vx_t)\bigr\|_2
\le \frac{\varepsilon \sqrt{d_z}}{1-\|\mathbf{K}\|_2}.
\end{equation}

\end{corollary}


In simple terms, \textbf{Theorem} \ref{thm:error_bound} states that if the learned observables are nearly Koopman-invariant (up to an error), then the multi-step prediction error grows in a controlled manner with the number of steps $\tau$.
\textbf{Corollary} \ref{thm:asymptotic_error_bound} further states that, for long-term prediction ($\tau\to\infty$), this error approaches a finite limit. These bounds, therefore, implies bounded multi-step prediction errors, meaning that the predictions do not diverge over time; thus, they motivate our spectral regularization and support the stability of NKM's long-horizon forecasts. For further theoretical foundations for studying dynamical systems through linear operators, see \cite{Mezic2013KoopmanReview,korda2018convergence,HornJohnson2012MatrixAnalysis,Chen1999LinearSystems}.

In \textbf{Sec.} \ref{sec:convergence}, we provide convergence and error bounds for the training of NKM by specializing existing nonconvex block-coordinate descent results and Koopman approximation theory to our setting.

\begin{figure*}[ht]
    \centering
    \includegraphics[width=1\textwidth]{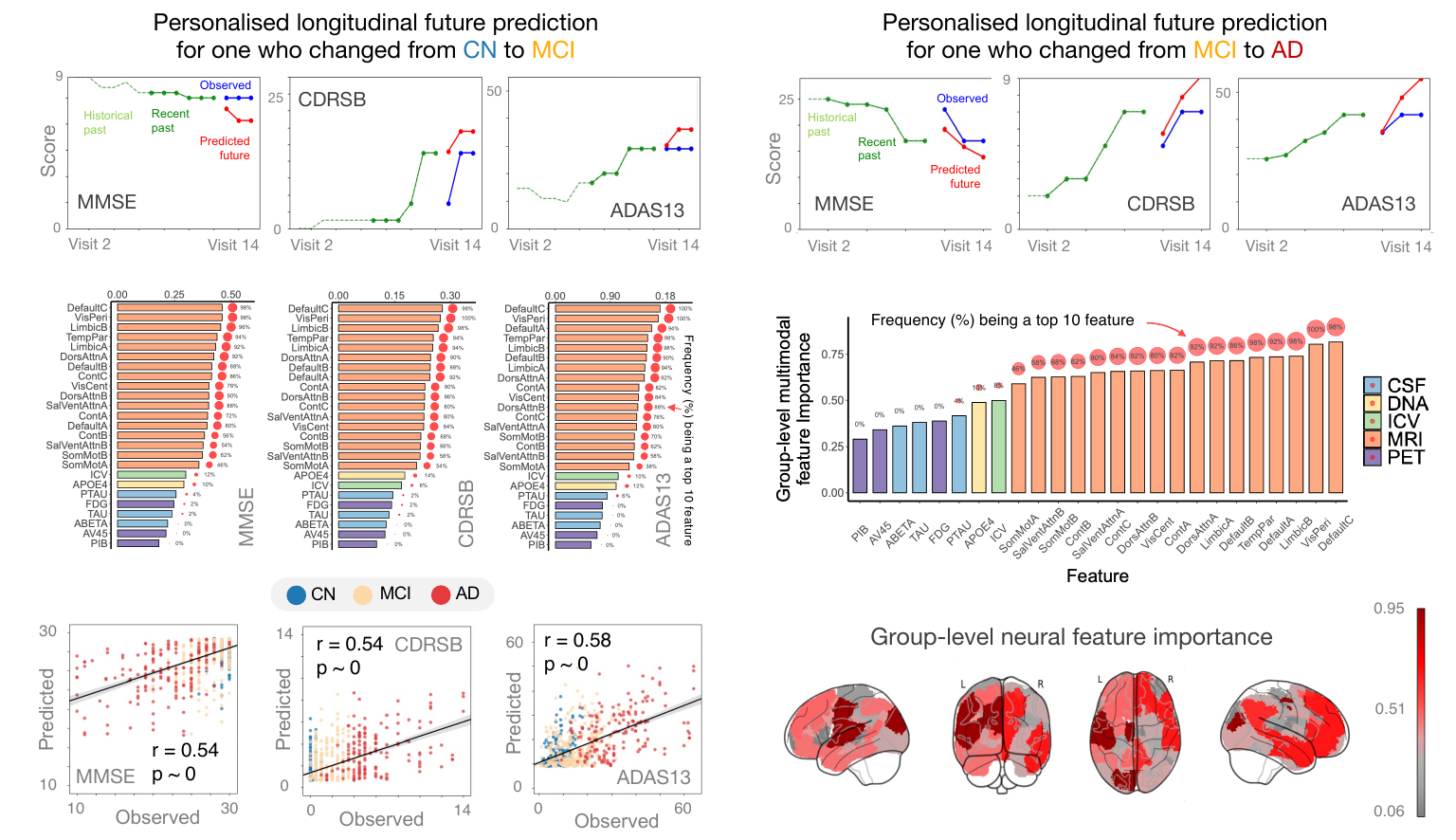}
    \caption{ \textbf{Using Neural Koopman Machine to forecast personalized future disease trajectory of Alzheimer's disease.} The top panels illustrate personalized longitudinal predictions for two patients—one progressing from cognitively normal (CN) to mild cognitive impairment (MCI), and the other from MCI to Alzheimer's disease (AD). Colored trajectories depict historical (light green), recent (dark green), predicted future (red), and observed future (blue) values for three cognitive assessments: MMSE, CDRSB, and ADAS13. The middle-left panel shows feature importance and selection consistency for each of the three cognitive scores, while the middle-right panel displays their aggregate, group-level counterparts (i.e., feature weights averaged across the three scores and 50 runs). In the bar–circle visualization, bar length denotes feature importance and circle size reflects the frequency with which a feature appears among the top 10 features across runs. The lower-left panel reports predictive performance for the three cognitive scores, and the lower-right panel depicts neural feature importance across the brain.}
    \label{fig:neural_koopman_analysis}
\end{figure*}

\section{Results}
\label{sec:experiments}

\subsection{Data and Experiment Setting}
\label{subsub:expsetup}
To evaluate NKM's efficacy, we applied it to analyze data from the Alzheimer's Disease Neuroimaging Initiative (ADNI) \cite{weiner_alzheimers_2012}. Specifically, we aimed to forecast multiple AD-related cognitive scores using historical multimodal data.

The multimodal data include 44 features per time step: 18 MRI features measuring Intracranial Volume (ICV) and cortical thickness of 202 brain regions organized by the Schaefer's 17-network parcellation (i.e., 17 cortical thickness features plus ICV), PET imaging (FDG, PIB, and AV-45), genetic (APOE4), CSF (Amyloid-beta (A$\beta$), t-tau, and p-tau protein), and demographics (two numeric and 17 categorical features). The longitudinal outcomes consist of three cognitive impairment measures, including Mini-Mental State Examination (MMSE), Clinical Dementia Rating - Sum of Boxes (CDRSB), and ADAS-Cog13 (ADAS13) scores.

For sequence modeling, we considered subjects with at least four time points and excluded subjects with fewer than four time points. 
This results in final sequences from 949 patients. For each subject, we constructed \textcolor{black}{temporal windows of width three} from consecutive visits (see \textbf{Sec.} \ref{sec:window_sizes} for a discussion on window sizes). \textcolor{black}{Once the NKM model is trained, for each temporal window we use features from the first three time points as input to predict the outcomes at the fourth time point, and the observed outcomes from the fourth time point are used to evaluate forecasting performance} (see \textbf{Fig.} \textcolor{blue}{\ref{fig:neural_koopman_analysis}}).

To ensure that no data leakage occurred during training and evaluation, we employed a subject-stratified five-fold cross-validation procedure. To handle missing values, we used k-nearest neighbors (KNN) imputation within each fold and standardized all features using parameters estimated exclusively from the training data, again to avoid leakage.
To evaluate the forecasting performance, we used Pearson's and Spearman's correlation coefficients as well as mean absolute error (MAE) and root mean square error (RMSE) between the predicted and observed scores at the next visit, averaged across the three cognitive scores. We trained each model for up to 300 epochs with a batch size of 64; convergence, however, typically occurred earlier due to the early stopping criteria in the training pipeline. We report the forecasting performance using NKM and baseline methods in \textbf{Table} \textcolor{blue}{\ref{tab:metrics_summary}}.

\begin{figure*}[ht]
    \centering
    \includegraphics[width=1\textwidth]{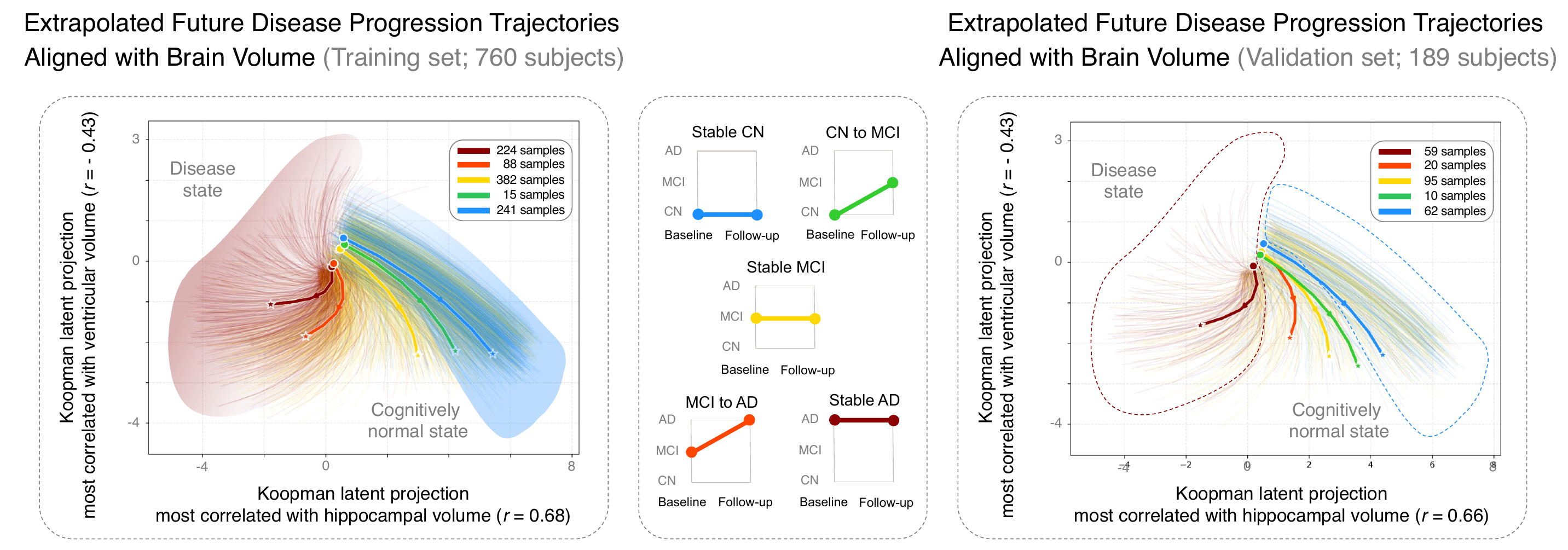}
    \caption{\textbf{The Neural Koopman Machine (NKM) uncovers unsupervised, neurobiologically grounded dynamics underlying Alzheimer's disease (AD).} 
This visualization maps the model's predictions onto a coordinate system that mirrors biology: the x-axis aligns with hippocampal health, and the y-axis with ventricular enlargement. It is important to note that, while we used brain volumes to orient these axes, the data points are generated solely by the NKM model's internal representations. To generate these trajectories, we used a sliding window approach. The model observes a patient for three time points and then projects their likely progression for the next five steps. Because we slide this window forward one step at a time, a single subject contributes multiple overlapping trajectories(samples) as they move through the study, as noted in the figure. In the plots, thick lines show the average path for each group, while thin lines show individual variation. The shaded regions \textbf{left} represent the density of stable CN and AD subjects in the training set. These same boundaries are replicated as dashed outlines in the validation plot \textbf{right} to serve as a fixed reference. This highlights the model's consistency, showing that validation trajectories fall primarily within the stable zones defined during training, despite minor distributional shifts.}
    \label{fig:koopman_interpretability}
\end{figure*}

\subsection{Model Comparisons}

To assess the efficacy of the NKM, we compared it against baseline longitudinal time series models, including linear regression, gradient boosting methods (e.g., XGBoost), dynamical modeling approaches (e.g., EDMD), recurrent neural networks (e.g., LSTM and GRU-D \cite{che2018recurrent}), and specialized time-series models (e.g., TimesNET \cite{wu_timesnet_2023}). \textbf{Table} \textcolor{blue}{\ref{tab:metrics_summary}} presents the comparison results.

\begin{table*}[!t]
\centering
\small 
\caption{Model comparisons (mean $\pm$ std over 5 folds). Higher is better for $r$ and $\rho$; lower is better for MAE and RMSE.}
\label{tab:metrics_summary}
\setlength{\tabcolsep}{8pt} 
\renewcommand{\arraystretch}{1.1} 
\begin{tabular}{lcccc}
\hline
\textbf{Model} & \textbf{Pearson $r$} & \textbf{Spearman $\rho$} & \textbf{MAE} & \textbf{RMSE} \\
\hline
\textbf{Neural Koopman Machine (NKM)}             & {\boldmath$0.5397 \pm 0.0237$} & {\boldmath$0.5024 \pm 0.0262$} & {\boldmath$3.8810 \pm 0.3710$} & $5.676 \pm 0.549$ \\
LSTM                                              & $0.5070 \pm 0.0309$ & $0.4856 \pm 0.0237$ & $4.0359 \pm 0.3380$ & {\boldmath$5.5955 \pm 0.5097$} \\
AttnRNN                                           & $0.5068 \pm 0.0240$ & $0.4810 \pm 0.0187$ & $4.0933 \pm 0.2877$ & $5.6918 \pm 0.4038$ \\
TimesNetTiny                                      & $0.4983 \pm 0.0255$ & $0.4729 \pm 0.0265$ & $5.1145 \pm 1.9693$ & $6.7831 \pm 2.0315$ \\
XGBoost                                           & $0.4926 \pm 0.0247$ & $0.4727 \pm 0.0121$ & $4.1345 \pm 0.2760$ & $5.6461 \pm 0.4888$ \\
RandomForest                                      & $0.4829 \pm 0.0259$ & $0.4741 \pm 0.0108$ & $4.1919 \pm 0.2766$ & $5.6748 \pm 0.4862$ \\
LinearRegression                                  & $0.4750 \pm 0.0377$ & $0.4786 \pm 0.0204$ & $4.2613 \pm 0.2695$ & $5.6804 \pm 0.4709$ \\
Neural ODE                                  & $0.3784 \pm 0.0501$ & $0.4099 \pm 0.0736$ & $4.667 \pm 0.220$ & $6.673 \pm 0.399$ \\
EDMD (RBF basis)                                  & $0.3706 \pm 0.0607$ & $0.4117 \pm 0.0344$ & $4.539 \pm 0.344$ & $6.413 \pm 0.642$ \\
\hline
\end{tabular}
\end{table*}

Our results show that NKM outperforms the baseline models across Pearson's $r$, Spearman's $\rho$, and MAE, with only its RMSE being slightly higher than that of LSTM and XGBoost. This suggests that, relative to existing longitudinal models, the NKM is effective at predicting multivariate future disease outcomes using past multimodal data.

\subsection{Unveiling Disease-Progression Dynamics in the Koopman Latent Space}

The brain can be regarded as a dynamical system. One of the goals of the paper is to characterize the brain's dynamical properties as it begins to develop cognitive impairment and Alzheimer's disease. Therefore, it is important to unveil the disease trajectory over time.

NKM uncovers the disease dynamics in its latent space. More specifically, the latent Koopman representations not only encode distinctly clustered features unique to each group (CN vs. MCI vs. AD), but also depict transition dynamics between groups (e.g., transitions from MCI to AD) (see \textbf{Fig.} \ref{fig:koopman_interpretability}). The latent features show distinct paths for stable groups (CN $\rightarrow$ CN, MCI $\rightarrow$ MCI, and AD $\rightarrow$ AD) and conversion trajectories (CN $\rightarrow$ MCI and MCI $\rightarrow$ AD). The latent representations also highlight important states and pathways: a CN state, an AD state, an early conversion pathway, and a disease progression pathway (see blue-shaded areas, red-shaded areas, the directed green line, and the directed orange line, respectively in \textbf{Fig.} \ref{fig:koopman_interpretability}). Subjects entering a CN state whose directed pathways also fall in the CN state are likely to maintain normal cognitive function. Subjects entering an AD state whose directed pathways also fall in the AD state are likely to have AD. Subjects entering areas around the early conversion pathway and whose trajectories are along this pathway are likely to begin developing mild cognitive impairment. Subjects entering areas around the disease progression pathway and whose trajectories are along this pathway are likely to convert into AD or MCI due to AD. Further, the latent Koopman states not only uncover the disease dynamics but also appear to be associated with brain volume changes. Particularly, the Koopman representations most associated with hippocampal and ventricular volumes help reveal how brain volume changes may give rise to disease transitions (see the right panel of \textbf{Fig.} \textcolor{blue}{\ref{fig:koopman_interpretability}}).Although structural volumetric markers were provided as inputs, the NKM's latent trajectories showed particularly strong alignment with hippocampal and ventricular changes, suggesting these regions dominate the learned disease dynamics. 

\subsection{Ablation Studies} 

\begin{table*}[!t]
\centering
\small
\caption{Ablation study for Neural Koopman Machine (NKM) and its full setup (mean Pearson $r \pm$ std over 5 folds). T (True) denotes the presence of a specific NKM component, while F (False) denotes its absence. 
}
\label{tab:nkm_ablation}
\setlength{\tabcolsep}{8pt}
\renewcommand{\arraystretch}{1.1}
\begin{tabular}{lccccc}
\hline
\textbf{Setup} & \textbf{Pearson $r$} & \textbf{Control} & \textbf{Temporal Attn.} & \textbf{Feature Attn.} & \textbf{Spectral Reg.} \\
\hline
\textbf{Full Neural Koopman Machine} & \boldmath$0.5397 \pm 0.0237$ & T & T & T & T \\
No control vector & $0.4179 \pm 0.0148$ & F & T & T & T \\
No temporal attention & $0.4908 \pm 0.0375$ & T & F & T & T \\
No feature attention & $0.5252 \pm 0.0226$ & T & T & F & T \\
No spectral regularization & $0.5101 \pm 0.0278$ & T & T & T & F \\
\hline
\end{tabular}
\end{table*}

To evaluate the contributions of individual components in our Neural Koopman Model (NKM), we performed a comprehensive ablation study using subject-level five-fold cross-validation on the ADNI dataset. We measured the performance using the mean Pearson correlation coefficient $r$ across the three cognitive targets (MMSE, CDRSB, and ADAS13 scores), with results reported as mean $\pm$ standard deviation over folds. The full NKM integrates a feature-aware encoder for grouped inputs (\textcolor{black}{genetic, proteomic, PET imaging, MRI, and demographic features}), a hierarchical attention control mechanism (combining temporal and feature attentions), and spectral regularization to enforce contractive dynamics via a bounded spectral radius $\rho = 0.95$. Ablations systematically removed one component at a time while retaining the others.

\textbf{Table} \textcolor{blue}{\ref{tab:nkm_ablation}} summarizes the results for the ablation studies. The full model achieves $r = 0.5397 \pm 0.0237$, establishing a strong baseline for longitudinal cognitive prediction. Removing the control vector that injects subject-specific latent adjustments via hierarchical attention results in the largest performance drop $r = 0.4179 \pm 0.0148$, ($\Delta r = -0.1218$), underscoring its critical role in capturing inter-subject variability and non-autonomous dynamics in neurodegenerative progression. Without temporal attention, performance degrades to $r = 0.4908 \pm 0.0375$ ($\Delta r = -0.0489$), highlighting the importance of modeling temporal dependencies across the three-visit window. Feature attention ablation results in a milder drop $r = 0.5252 \pm 0.0226$ ($\Delta r = -0.0145$), as the grouped encoder already provides some disentanglement, but feature attention still aids in prioritizing salient biomarkers (e.g., genetics over demographics). Finally, disabling spectral regularization that penalizes the Koopman operator's spectral norm beyond $\rho$ via $\eta (\max(0, \|\mathbf{K}\|_2^2 - \rho^2)^2$ with $\eta=0.01$ results in $r = 0.5101 \pm 0.0278$ ($\Delta r = -0.0296$), emphasizing \textcolor{black}{the role} of stable, contractive evolution for multi-step rollouts.

These ablation studies confirm the synergistic contributions of NKM's design: the control vector and temporal modeling drive the majority of gains ($\sim$29\% and $\sim$10\% relative improvements, respectively), while feature attention and spectral constraints enhance robustness. 
In summary, our ablation studies demonstrate NKM's modularity and the pivotal role of \textcolor{black}{$\alpha \times \beta$-knowledge-assisted} Koopman dynamics in modeling disease progression.

\subsection{Biological and Neuroanatomical Substrates of Longitudinal Cognitive Decline via Interpretability Analysis}
Here, we examine the importance and consistency of the biomarkers identified by NKM for longitudinal personalized AD forecasting. 
First, NKM uncovers brain signatures respectively predictive of three cognitive scores over time (see \textbf{Fig.} \ref{fig:neural_koopman_analysis}). Our results indicate that longitudinal AD score prediction involves distributed cortical regions, including frontal, parietal, temporal, and occipital areas. Consistently, regions within the default mode network (DMN), visual periphery, limbic system, and temporal–parietal areas were identified as key neural biomarkers across all three scores. These findings not only align with evidence that widespread cortical changes in AD are associated with global cognitive decline \cite{bakkour_effects_2013,pereira_disrupted_2016}, but also highlight that the patterns captured by these distributed changes provide valuable information for forecasting AD progression over time. The limbic regions appear informative for predicting both MMSE and CDRSB scores, consistent with functional decline and limbic vulnerability in early AD \cite{whitwell_imaging_2018}. In contrast, the DMN \cite{buckner_molecular_2005,jones_age-related_2011}, particularly \textit{DefaultA} 
\cite{rieck2021dataset},
consisting of parts of BA39 and BA10, is especially predictive of the ADAS13 score over time. This both aligns with previous findings on the association between BA39 \cite{howlett2013clusterin}, BA10\cite{wu2011altered}, and AD, and underscores its relevance for forecasting longidutinal cognitive decline. 

Second, NKM identified APOE4 as a useful AD biomarker. While AD is recognized as a polygenic disease~\cite{harrisonPolygenicScoresPrecision2020}, the APOE4 allele remains the most well-established genetic risk factor for late-onset AD~\cite{bellenguezGeneticsAlzheimersDisease2020a}. Third, our results suggest that CSF features, particularly p-tau, are useful for predicting longitudinal MMSE, CDRSB, and ADAS13 scores, consistent with the central role of tau pathology in AD progression \cite{zhang_recent_2024}. Fourth, PET images, particularly FDG-PET, provide predictive value for AD forecasting, as FDG-PET measures glucose metabolism in the brain, which is typically reduced in AD \cite{chetelat2020amyloid}. Fifth, our results suggest that ICV may be valuable for predicting AD. Although the relationship between ICV and AD progression is not yet fully established, individuals with larger brains are thought to have greater ``brain reserve'', potentially allowing them to tolerate more pathology before exhibiting cognitive symptoms \cite{van2018intracranial}. Finally, we find that the biomarkers with both high feature importance scores and selection stability (i.e., consistently ranked among the top 10 features across 50 runs) are largely MRI-derived. This is likely attributable to the substantially greater availability of longitudinal MRI data in our dataset (each subject having at least three MRI scans), which enables the model to capture temporal trajectories of disease progression reflected in the repeatedly measured scans.
In contrast, the availability of repeated measurements varies significantly by modality. While ICV data is available longitudinally for nearly all subjects (923 of 949), repeated measurements are sparse for \textbf{PET (214)} and \textbf{CSF (107)}, and absent for DNA. However, the framework presented in this paper, along with the division of $\alpha$- and $\beta$-knowledge, effectively manages this disparity, ensuring that the dense longitudinal data from MRI and ICV provide the essential information for modeling AD progression.

\section{Discussion}
The structural pipeline of the NKM model that we presented in this paper, and the division of $\alpha$- and $\beta$-knowledge, provides us with a robust machine learning framework for longitudinal assessment of multiple disease-related outcomes using multimodal data.

The modular design of the NKM allows the model to be run on datasets that do not include all modalities considered in the current study or other studies -- a common challenge in AD research and datasets. Another design decision that invites discussion is using a one-hot encoder rather than a categorical embedding for demographic features. This is a common trade-off in machine learning\cite{hancock2020survey}. Within the NKM, a one-hot encoding is used so that these categorical groups are fed into dedicated Multi-Layer Perceptrons (MLPs) within the Feature-Aware Encoder we proposed. One-hot encoding offers interpretability and simplicity, treating demographic categories as mutually orthogonal without assuming relationships. Since demographic data typically have low cardinality, this is a suitable mathematical procedure. Hence, we avoid mathematical overparameterization, and, consequently, the integration 
of demographics with the models' group-aware design is more seamless. This design choice offers two key advantages: (1) it maintains demographic feature independence without introducing representational bias, and (2) it ensures direct compatibility with the Koopman operator's state-space formulation. The alternative would be to use the categorical embeddings that learn dense vectors per category. This may capture more of the semantic similarities for better in-model interaction, and, within our use case, this may improve the gating mechanism, as the signals coming from the categories will be more prominent/obscure. However, given the low cardinality of our demographic data, categorical embeddings pose the risk of overfitting and training overhead, while reducing the interpretability of the model architecture \cite{Guo2016EntityEO,gorishniy2023revisitingdeeplearningmodels}. Therefore, we use one-hot encoding to handle demographic information, \textcolor{black}{which lacks complex semantics, and rely on the encoder's MLPs to provide equivalent compression.}

\section{Final remarks}
Our work has a few limitations. First, due partly to the data acquisition, we considered regularly sampled time series. Future work should evaluate the efficacy of NKM on irregularly sampled data. Second, our main focus here is to identify interpretable multimodal biomarkers predictive of future disease progression; independent work should evaluate whether these biomarkers are causally linked to AD progression. Third, while one can apply NKM to study other diseases, the trained architectures from AD data cannot yet be directly transferred to study datasets of other diseases. Future work can explore the potential of NKM for forecasting other disease trajectories. Lastly, recent studies have suggested promising blood-based biomarkers (e.g., p-tau181/217, neurofilament light chain (NfL), and glial fibrillary acidic protein (GFAP)) \cite{grande_blood-based_2025}; while not included in our study, they offer new avenues to improve NKM by incorporating additional biomarkers.

\section{Acknowledgements}
G. Hrusanov provided all the mathematical formulations and the numerical results. 

This work was jointly supported by F. Hoffmann - La Roche Ltd. and the Swiss National Science Foundation (SNSF) under Grants \texttt{CR00I5-235987} and \texttt{3200-0\_239967}.


K. Koscielska is an employee of Roche Diagnostics International Ltd., Rotkreuz, 6343, Switzerland and C. Magnus is an employee of F. Hoffmann - La Roche Ltd., Roche Information Solutions, Basel, 4058, Switzerland. K. Koscielska and C. Magnus hold shares in F. Hoffmann-La Roche Ltd.



Data collection and sharing for this project was funded by the Alzheimer's Disease Neuroimaging Initiative
(ADNI) (National Institutes of Health Grant U01 AG024904) and DOD ADNI (Department of Defense award
number W81XWH-12-2-0012). ADNI is funded by the National Institute on Aging, the National Institute of
Biomedical Imaging and Bioengineering, and through generous contributions from the following: AbbVie,
Alzheimer's Association; Alzheimer's Drug Discovery Foundation; Araclon Biotech; BioClinica, Inc.; Biogen;
Bristol-Myers Squibb Company; CereSpir, Inc.; Cogstate; Eisai Inc.; Elan Pharmaceuticals, Inc.; Eli Lilly and
Company; EuroImmun; F. Hoffmann-La Roche Ltd and its affiliated company Genentech, Inc.; Fujirebio; GE
Healthcare; IXICO Ltd.; Janssen Alzheimer Immunotherapy Research \& Development, LLC.; Johnson \&
Johnson Pharmaceutical Research \& Development LLC.; Lumosity; Lundbeck; Merck \& Co., Inc.; Meso
Scale Diagnostics, LLC.; NeuroRx Research; Neurotrack Technologies; Novartis Pharmaceuticals
Corporation; Pfizer Inc.; Piramal Imaging; Servier; Takeda Pharmaceutical Company; and Transition
Therapeutics. The Canadian Institutes of Health Research is providing funds to support ADNI clinical sites
in Canada. Private sector contributions are facilitated by the Foundation for the National Institutes of Health
\url{www.fnih.org}. The grantee organization is the Northern California Institute for Research and Education,
and the study is coordinated by the Alzheimer's Therapeutic Research Institute at the University of Southern
California. ADNI data are disseminated by the Laboratory for Neuro Imaging at the University of Southern
California.



\bibliographystyle{IEEEtran}
\bibliography{all_refs.bib}


\clearpage

\newpage

\section{Supplementary Materials}
\vspace{5mm}
\noindent Overview of the Supplementary Materials:
\vspace{2mm}
\begin{itemize}[label={}]
  \item \textbf{A.} Additional Theoretical Results.
  \item \textbf{B.} Proof of Theorem 1.
  \item \textbf{C.} Neural Koopman Operator Learning and Optimization.
  \item \textbf{D.} Experimental Details and Implementation.
  \item \textbf{E.} Abbreviations and Terminology.
\end{itemize}

\subsection{Additional Theoretical Results}

\subsubsection{Convergence of gradient descent} \label{sec:convergence}

\begin{theorem}[Convergence of alternating gradient descent]
\label{thm:convergence}
Let $\mathcal{L}(\theta_{\mathrm{enc}}, \mathbf{K}) 
  = \mathcal{L}_{\mathrm{pred}} + \lambda \mathcal{L}_{\mathrm{koop}}$
be the composite loss, with encoder parameters
$\theta_{\mathrm{enc}}$ and Koopman matrix 
$\mathbf{K} \in \mathbb{R}^{d_z \times d_z}$.
Assume:
\begin{enumerate}[label=(\roman*)]
\item The encoder $\Phi_{\theta_{\mathrm{enc}}} : \mathcal{X} \to \mathbb{R}^{d_z}$ is
$L_\Phi$-Lipschitz continuous in $\theta_{\mathrm{enc}}$ for all $x \in \mathcal{X}$.
\item For every fixed $\mathbf{K}$, 
$\mathcal{L}(\theta_{\mathrm{enc}}, \mathbf{K})$ is $L_\theta$-smooth in 
$\theta_{\mathrm{enc}}$.
\item For every fixed $\theta_{\mathrm{enc}}$, 
$\mathcal{L}_{\mathrm{koop}}(\theta_{\mathrm{enc}}, \mathbf{K})$ is
$L_K$-smooth and $\mu$-strongly convex in $\mathbf{K}$, and the
$\mathbf{K}$-updates include projection onto
$\mathcal{B}_\rho = \{\mathbf{K} : \|\mathbf{K}\|_2 \le \rho\}$ for some $\rho < 1$.
\item The step sizes satisfy
$\alpha_\theta \le 1/L_\theta$ and
$0 < \alpha_K < 2/(L_K + \mu)$.
\end{enumerate}
Consider the alternating updates
\begin{align*}
\theta^{(k+1)} 
&= \theta^{(k)} 
 - \alpha_\theta \nabla_{\theta_{\mathrm{enc}}} \mathcal{L}(\theta^{(k)}, \mathbf{K}^{(k)}),\\
\mathbf{K}^{(k+1)} 
&= \Pi_{\mathcal{B}_\rho}\!\left(
\mathbf{K}^{(k)} 
 - \alpha_K \nabla_{\mathbf{K}} \mathcal{L}_{\mathrm{koop}}
(\theta^{(k+1)}, \mathbf{K}^{(k)})
\right).
\end{align*}
Then there exists a constant $c > 0$ (depending on $L_\theta$, $L_K$, $\mu$, and $L_\Phi$) such that
\[
\mathcal{L}(\theta^{(k+1)}, \mathbf{K}^{(k+1)}) 
\le \mathcal{L}(\theta^{(k)}, \mathbf{K}^{(k)}) 
 - c \,\|\nabla \mathcal{L}(\theta^{(k)}, \mathbf{K}^{(k)})\|^2,
\]
and, in particular, for all $K \ge 1$,
\[
\min_{0 \le k < K}
\|\nabla \mathcal{L}(\theta^{(k)}, \mathbf{K}^{(k)})\|^2
\le \frac{C}{K}
\]
for some constant $C>0$. Thus, an $\epsilon$-approximate first-order stationary point
is reached after $K = \mathcal{O}(\epsilon^{-1})$ iterations.
\end{theorem}


In words, \textbf{Theorem} \ref{thm:convergence} asserts that if the encoder and Koopman matrix satisfy certain regularity conditions (e.g., smoothness and spectral-norm constraint), then alternating gradient descent converges. Additionally, after sufficient iterations, the updates reach an $\epsilon$-approximate stationary point, where the loss cannot be further significantly decreased \cite{Beck2017FirstOrder,BoydVandenberghe2004ConvexOptimization}.

\begin{proof}
\label{Proof:thm2}
We apply block coordinate descent convergence analysis. Define iterates $(\theta^{(k)}, \mathbf{K}^{(k)})$ with alternating updates:
\begin{align}
\theta^{(k+1)} &= \theta^{(k)} - \alpha_\theta \nabla_\theta \mathcal{L}(\theta^{(k)}, \mathbf{K}^{(k)}), \\
\mathbf{K}^{(k+1)} &= \Pi_{\mathcal{B}_\rho}\left(\mathbf{K}^{(k)} - \alpha_K \nabla_{\mathbf{K}} \mathcal{L}(\theta^{(k+1)}, \mathbf{K}^{(k)})\right),
\end{align}
where $\Pi_{\mathcal{B}_\rho}(\cdot)$ projects onto $\mathcal{B}_\rho = \{\mathbf{K} : \|\mathbf{K}\|_2 \leq \rho\}$.

\textit{Step 1 (Encoder subproblem):} For fixed $\mathbf{K}$, $L_\Phi$-Lipschitz continuity of $\Phi_\theta$ implies $L$-smoothness of $\mathcal{L}(\theta, \mathbf{K}^{(k)})$ in $\theta$. Standard SGD analysis yields:
\begin{equation}
\mathcal{L}(\theta^{(k+1)}, \mathbf{K}^{(k)}) - \mathcal{L}(\theta^{(k)}, \mathbf{K}^{(k)}) \leq -\frac{\alpha_\theta}{2}\|\nabla_\theta \mathcal{L}\|^2.
\end{equation}

\textit{Step 2 (Koopman subproblem):} For fixed $\theta$, $\mathcal{L}_{\text{koop}}$ is quadratic in $\mathbf{K}$ (cf. Eq.~40), with minimum at $\mathbf{K}^* = (\hat{\mathbf{C}}_{z'z} - \hat{\mathbf{C}}_{cz})\hat{\mathbf{C}}_{zz}^{-1}$ (Eq.~47). The projection satisfies:
\begin{equation}
\|\mathbf{K}^{(k+1)} - \mathbf{K}^*\|_F \leq (1 - \mu \alpha_K)\|\mathbf{K}^{(k)} - \mathbf{K}^*\|_F,
\end{equation}
where $\mu > 0$ is the strong convexity constant of $\mathcal{L}_{\text{koop}}$.

\textit{Step 3 (Coupling bound):} The change in $\theta$ induces error in the Koopman fixed point. By Lipschitz continuity of $\Phi_\theta$:
\begin{equation}
\|\mathbf{K}^*(\theta^{(k+1)}) - \mathbf{K}^*(\theta^{(k)})\|_F \leq C L_\Phi \|\theta^{(k+1)} - \theta^{(k)}\|,
\end{equation}
for some constant $C$ depending on data moments.

\textit{Step 4 (Descent property):} Combining Steps 1--3 and applying Lemma 2.3 of \cite{doi:10.1137/120891009}, the composite loss satisfies sufficient descent:
\begin{equation}
\mathcal{L}(\theta^{(k+1)}, \mathbf{K}^{(k+1)}) \leq \mathcal{L}(\theta^{(k)}, \mathbf{K}^{(k)}) - c\|\nabla \mathcal{L}\|^2,
\end{equation}
for $c > 0$ when $\alpha_\theta$ and $\alpha_K$ satisfy stated bounds. Telescoping over $K$ iterations and applying $L$-smoothness gives $\min_{k \leq K} \|\nabla \mathcal{L}\|^2 = \mathcal{O}(K^{-1})$. \hfill
\end{proof}

\subsubsection{Koopman Mode Decomposition for Nonlinear Systems}
For nonlinear systems with a discrete spectrum, the Koopman eigenfunction decomposition provides a powerful framework for analysis. Let $\{\varphi_j\}$ be the Koopman eigenfunctions satisfying:
\begin{equation*}
(\mathcal{K}\varphi_j)(\vx) = \lambda_j \varphi_j(\vx),
\end{equation*}
and let $\{\vv_j\}$ be the corresponding Koopman modes for the full state observable $g(\vx) = \vx$.

Then, the state evolution can be expressed as:
\begin{equation*}
\vx_t = \sum_{j} \varphi_j(\vx_0) \lambda_j^t \vv_j.
\end{equation*}
This decomposition separates the dynamics into modes with distinct temporal behaviors determined by the eigenvalues $\lambda_j$. In our Neural Koopman Machine framework, we approximate these eigenfunctions and modes through the learned representations. We should explicitly state that:

\begin{itemize}
    \item This assumes $\mathcal{K}$ has a discrete spectrum, which is not true for all operators.
    \item For finite-dimensional approximations, this becomes approximate.
    \item The modes $v_j$ are generally not orthogonal.
\end{itemize}

\subsubsection{Duality with Perron-Frobenius Operator}
The Koopman operator $\mathcal{K}$ is dual to the Perron-Frobenius operator $\mathcal{P}$, which propagates probability densities:
\begin{equation*}
(\mathcal{P}\rho)(\vy) = \int_{\{\vx : \vf(\vx) = \vy\}} \rho(\vx) \left|\det\left(\frac{\partial \vf}{\partial \vx}\right)^{-1}\right| d\vx,
\end{equation*}
where $\rho$ is a probability density function.

The duality relation is:
\begin{equation*}
\int_{\sM} g(\vy)(\mathcal{P}\rho)(\vy)d\vy = \int_{\sM} (\mathcal{K}g)(\vx)\rho(\vx)d\vx,
\end{equation*}
for all observables $g$ and densities $\rho$.

This duality connects our Koopman framework with probabilistic models of dynamical systems.

\subsection{Proof for Theorem 1}

\subsubsection{Proof for Theorem \ref{thm:error_bound} for Error Analysis for Multi-Step Prediction}

\label{Proof:thm1 and corollary1}
\hfill\\

Assume:
\[
\|\mathbf{K}\|_2 < 1,
\]
which implies, all eigenvalues lie within the unit circle. 
This assumption ensures that the autonomous Koopman dynamics are contractive, guaranteeing bounded long-term prediction errors. However, Alzheimer's disease exhibits progressive decline rather than convergence to a stable state. In our model, the Koopman matrix $K$ captures \emph{temporal variations and short-term fluctuations} around the disease trajectory, while the control term $c_{\text{control}}$ models the \emph{systematic progression and subject-specific drift} in cognitive decline. This decomposition allows us to maintain theoretical stability guarantees while accurately capturing progressive disease dynamics through the additive control mechanism.

Moreover, assume either (i) $\mathbf{K}$ is diagonalizable, or (ii) every Jordan block
satisfies the same bound. Otherwise $\|\mathbf{K}^n\|_2$ may grow polynomially in $n$, invalidating the
geometric bound.
\vspace{2mm}
\paragraph{One-Step Error Bound}
\hfill\\
Assume that for all $\vx \in \sA \subset \sM$ (an attractor of the dynamical system):
\begin{equation*}
\biggl\|\mathcal{K}\phi_j(\vx) - \sum_{i=1}^{d_z} K_{ij}\,\phi_i(\vx)\biggr\|_2 \leq \varepsilon_j,
\end{equation*}
with $\varepsilon = \max_j \varepsilon_j$ assuming uniform $\varepsilon_j = \varepsilon$ for all~$j$. Then the one-step lifting error satisfies:
\begin{equation*}
\bigl\|\boldsymbol{\Phi}(f(\vx_t))-\mathbf{K}\,\boldsymbol{\Phi}(\vx_t)\bigr\|_2 \leq \varepsilon\sqrt{d_z},
\end{equation*}
where the factor $\sqrt{d_z}$ accounts for the dimension of the lifted space.

For simplicity, denote $\tilde{\varepsilon} = \varepsilon\sqrt{d_z}$ and assume:
\begin{equation*}
\bigl\|\boldsymbol{\Phi}(f(\vx_t))-\mathbf{K}\,\boldsymbol{\Phi}(\vx_t)\bigr\|_2 \leq \tilde{\varepsilon}.
\end{equation*}

\vspace{2mm}
\paragraph{Two-Step Error Propagation}
\hfill\\
For the two-step prediction, we have:
\begin{align*}
&\|\boldsymbol{\Phi}(\vx_{t+2})-\mathbf{K}^2\boldsymbol{\Phi}(\vx_t)\|_2
= \|\boldsymbol{\Phi}(f(\vx_{t+1}))-\mathbf{K}^2\boldsymbol{\Phi}(\vx_t)\|_2 \\
&= \|\boldsymbol{\Phi}(f(\vx_{t+1}))-\mathbf{K}\boldsymbol{\Phi}(\vx_{t+1})
+ \mathbf{K}\boldsymbol{\Phi}(\vx_{t+1})-\mathbf{K}^2\boldsymbol{\Phi}(\vx_t)\|_2 \\
&\le \|\boldsymbol{\Phi}(f(\vx_{t+1}))-\mathbf{K}\boldsymbol{\Phi}(\vx_{t+1})\|_2
+ \|\mathbf{K}\boldsymbol{\Phi}(\vx_{t+1})-\mathbf{K}^2\boldsymbol{\Phi}(\vx_t)\|_2 \\
&\le \tilde{\varepsilon} + \|\mathbf{K}\|_2\,\|\boldsymbol{\Phi}(\vx_{t+1})-\mathbf{K}\boldsymbol{\Phi}(\vx_t)\|_2
\le \tilde{\varepsilon} + \|\mathbf{K}\|_2\,\tilde{\varepsilon}\\
&= \tilde{\varepsilon}(1+\|\mathbf{K}\|_2).
\end{align*}

\vspace{2mm}
\paragraph{$\tau$-Step Error Bound)}
\hfill\\
By induction, we derive the $n$-step prediction error bound:
\begin{equation}\label{eq:geom_error}
\bigl\|\boldsymbol{\Phi}(\vx_{t+\tau}) - \mathbf{K}^\tau\boldsymbol{\Phi}(\vx_t)\bigr\|_2
\leq
\tilde{\varepsilon}\textstyle\sum_{i=0}^{\tau-1}\|\mathbf{K}\|_2^i
= \tilde{\varepsilon}\frac{1-\|\mathbf{K}\|_2^\tau}{1-\|\mathbf{K}\|_2},
\end{equation}
provided $\|\mathbf{K}\|_2 < 1$. This convergent geometric bound is only valid when the spectral norm is strictly less than unity; divergent behavior emerges when $\|\mathbf{K}\|_2 \geq 1$, underscoring the necessity of our spectral regularization approach.

\begin{proof}
Following Koopman/Galerkin setting \cite{Mezic2013KoopmanReview,williams_data-driven_2015,korda2018convergence}, we proceed by induction. The base case for $\tau=1$ is our assumption:
\begin{equation*}
\|\boldsymbol{\Phi}(\vx_{t+1}) - \mathbf{K}\boldsymbol{\Phi}(\vx_t)\|_2 \leq \tilde{\varepsilon}.
\end{equation*}
Assuming the bound holds for step $\tau$, invoking geometric-series and matrix-power bounds with the operator norm used in standard matrix analysis \cite{HornJohnson2012MatrixAnalysis}, we have, for step $\tau+1$:
\begin{align*}
&\bigl\|\boldsymbol{\Phi}(\vx_{t+\tau+1}) - \mathbf{K}^{\tau+1}\boldsymbol{\Phi}(\vx_t)\bigr\|_2 \notag \\
&\quad= \bigl\|\boldsymbol{\Phi}(f(\vx_{t+\tau})) - \mathbf{K}^{\tau+1}\boldsymbol{\Phi}(\vx_t)\bigr\|_2 \\
&\quad= \bigl\|\boldsymbol{\Phi}(f(\vx_{t+\tau})) - \mathbf{K}\boldsymbol{\Phi}(\vx_{t+\tau}) \notag \\
&\qquad + \mathbf{K}\boldsymbol{\Phi}(\vx_{t+\tau}) - \mathbf{K}^{\tau+1}\boldsymbol{\Phi}(\vx_t)\bigr\|_2 \\
&\quad\leq \bigl\|\boldsymbol{\Phi}(f(\vx_{t+\tau})) - \mathbf{K}\boldsymbol{\Phi}(\vx_{t+\tau})\bigr\|_2 \notag \\
&\qquad + \bigl\|\mathbf{K}\boldsymbol{\Phi}(\vx_{t+\tau}) - \mathbf{K}^{\tau+1}\boldsymbol{\Phi}(\vx_t)\bigr\|_2 \\
&\quad\leq \tilde{\varepsilon} + \|\mathbf{K}\|_2 \cdot \bigl\|\boldsymbol{\Phi}(\vx_{t+\tau}) - \mathbf{K}^{\tau}\boldsymbol{\Phi}(\vx_t)\bigr\|_2 \\
&\quad\leq \tilde{\varepsilon} + \|\mathbf{K}\|_2 \cdot \tilde{\varepsilon}\frac{1-\|\mathbf{K}\|_2^\tau}{1-\|\mathbf{K}\|_2} \\
&\quad= \tilde{\varepsilon}\frac{1-\|\mathbf{K}\|_2^{\tau+1}}{1-\|\mathbf{K}\|_2}.
\end{align*}
\end{proof}

\paragraph{Asymptotic Error Bound (Corollary \ref{thm:asymptotic_error_bound})}
\hfill\\
For long-term predictions ($\tau \to \infty$), if $\|\mathbf{K}\|_2 < 1$, the error bound converges to:
\begin{equation*}
\lim_{\tau \to \infty} \bigl\|\boldsymbol{\Phi}(\vx_{t+\tau}) - \mathbf{K}^\tau\boldsymbol{\Phi}(\vx_t)\bigr\|_2 \leq \frac{\tilde{\varepsilon}}{1-\|\mathbf{K}\|_2},
\end{equation*}
ensuring stable long-term predictions with bounded error.

If $\|\mathbf{K}\|_2 \geq 1$, the error may grow unbounded with $\tau$, reflecting the challenge of predicting chaotic systems over extended periods. This motivates our spectral regularization approach (see \textbf{Sec.} \textcolor{purple}{\ref{subsec:spectral_reg}}).

\subsection{Neural Koopman Operator Learning and Optimization}

\subsubsection{Finite-Dimensional Koopman Approximation}
Consider the nonlinear dynamical system:
\begin{equation*}
  \vx_{t+1} \;=\; f(\vx_t), 
  \qquad 
  \vx_t \in \sM \subset \R^n,
\end{equation*}
where $f:\sM\to\sM$ is a nonlinear map on the state manifold $\sM$.  
The Koopman operator $\mathcal{K}$ is the composition operator:
\begin{equation*}
  (\mathcal{K}g)(\vx) \;=\; g\bigl(f(\vx)\bigr),
  \qquad g\in\sF,
\end{equation*}
acting on a suitable function space $\sF$ of observables.

Let a finite set of observables be collected in
$\boldsymbol{\Phi}:\sM\to\R^{d_z}$:
\begin{equation*}
  \boldsymbol{\Phi}(\vx)
  \;=\;
  \begin{bmatrix}
    \phi_1(\vx) \\[-2pt] \phi_2(\vx) \\[-4pt] \vdots \\[-4pt] \phi_{d_z}(\vx)
  \end{bmatrix},
  \qquad
  \vz_t \;=\; \boldsymbol{\Phi}(\vx_t).
\end{equation*}
We approximate the infinite-dimensional Koopman operator by a matrix $\mathbf{K}\in\R^{d_z\times d_z}$:
\begin{equation}
  \label{eq:koopman_approx_app}
  \boldsymbol{\Phi}\!\bigl(f(\vx_t)\bigr)
  \;=\;
  (\mathcal{K}\boldsymbol{\Phi})(\vx_t)
  \;\approx\;
  \mathbf{K}\,\boldsymbol{\Phi}(\vx_t).
\end{equation}
Component-wise this reads:
\begin{equation*}
  (\mathcal{K}\phi_j)(\vx)
  \;\approx\;
  \sum_{i=1}^{d_z} K_{ij}\,\phi_i(\vx).
\end{equation*}

\paragraph{Approximation quality.}
With $\|\cdot\|_{\sF}$ a norm on $\sF$, define
\begin{equation*}
  \varepsilon_j
  \;=\;
  \min_{\{a_{ij}\}}
  \Bigl\|
    \mathcal{K}\phi_j \;-\!
    \sum_{i=1}^{d_z} a_{ij}\phi_i
  \Bigr\|_{\sF} .
\end{equation*}
For tighter bounds on attractors like disease states (e.g., $\sA\subset\sM$), we use a uniform (pointwise Euclidean) bound:
\begin{equation*}
  \sup_{\vx\in\sA}
  \bigl\|
    \mathcal{K}\phi_j(\vx)
    -\textstyle\sum_{i=1}^{d_z} a_{ij}\phi_i(\vx)
  \bigr\|_2 ,
\end{equation*}
implicitly making the error estimate more restrictive.

\subsubsection{Extended Dynamic Mode Decomposition (EDMD)}
Given snapshots $\{(\vx_i,\vy_i)\}_{i=1}^m$
with $\vy_i=f(\vx_i)$, EDMD solves:
\begin{equation*}
  \min_{\mathbf{K}}
  \sum_{i=1}^m
    \bigl\|\boldsymbol{\Phi}(\vy_i)-\mathbf{K}\,\boldsymbol{\Phi}(\vx_i)\bigr\|_2^2 .
\end{equation*}

Denote:
\begin{equation}
\label{eq:XY_def}
\begin{aligned}
  \mX &=
  \begin{bmatrix}
    \boldsymbol{\Phi}(\vx_1)^\top\\
    \vdots\\
    \boldsymbol{\Phi}(\vx_m)^\top
  \end{bmatrix}
  \in \R^{m\times d_z},\\
  \mY &=
  \begin{bmatrix}
    \boldsymbol{\Phi}(\vy_1)^\top\\
    \vdots\\
    \boldsymbol{\Phi}(\vy_m)^\top
  \end{bmatrix}
  \in \R^{m\times d_z}.
\end{aligned}
\end{equation}

Then:
\begin{equation}
\label{eq:edmd_closed_form}
\mathbf{K}=\mY^{\top}\mX(\mX^{\top}\mX)^{\dagger},
\end{equation}
where ${}^\dagger$ is the Moore--Penrose pseudoinverse. In practice, $\mX^{\top}\mX$ is often rank-deficient due to limited sample sizes or colinearities in the data, requiring careful regularization.

\paragraph{Empirical \textit{Gramians}}
Define:
\begin{equation*}
  \mG := \tfrac1m\,\mX^{\top}\mX,
  \qquad
  \mA := \tfrac1m\,\mY^{\top}\mX .
\end{equation*}
With these \emph{forward} \textit{Gramians}, \textbf{Eq.} \textcolor{blue}{\ref{eq:edmd_closed_form}} becomes:
\begin{equation}
  \label{eq:edmd_solution_correct}
    \mathbf{K}
    \;=\;
    \mA\,\mG^{\dagger}
  .
\end{equation}

\paragraph{Regularized variant}
For numerical stability,
\begin{equation*}
  \mathbf{K}
  = \mA \bigl(\mG+\alpha\mI\bigr)^{-1},
\end{equation*}
unchanged by the above correction.

The remaining discussion on the joint learning of
$(\boldsymbol{\Phi},\mathbf{K})$ and on regularization carries through verbatim.

\vspace{2mm}
\subsubsection{Neural Parameterization of Observables and Koopman Matrix}
Our neural framework learns both the observable functions $\boldsymbol{\Phi}$ and the Koopman matrix $\mathbf{K}$ simultaneously.

\paragraph{Encoder Architecture}
The SiLU (Sigmoid Linear Unit) activation function, also known as the swish function, is defined as $\mathrm{SiLU}(x) = x \cdot \sigma(x)$ where $\sigma$ is the sigmoid function. This activation has been shown to outperform ReLU in deep networks due to its smooth, non-monotonic nature.

We implement $\boldsymbol{\Phi}$ as a feature-aware encoder followed by a shared integration network:
\begin{align*}
\vh_t^{(g)}
&= \mathrm{SiLU}\bigl(\mathrm{LayerNorm}(\mW_g\,\vx_t^{(g)}+\vb_g)\bigr), 
\quad g=1,\dots,G,\\
\vz_t
&= \boldsymbol{\Phi}(\vx_t; \vtheta_{\text{enc}})
= \mathrm{SiLU}\Bigl(\mW_{\mathrm{int}}
\bigl[\vh_t^{(1)},\dots,\vh_t^{(G)}\bigr]
+\vb_{\mathrm{int}}\Bigr),
\end{align*}
where $\vx_t^{(g)}$ are disjoint feature groups (projected to uniform $d_z$ after fusion), and LayerNorm is defined as:
\begin{equation*}
\begin{aligned}
\mathrm{LayerNorm}(\vu)
&= \boldsymbol{\gamma}\odot\frac{\vu-\mu}{\sqrt{\sigma^2+\epsilon}}+\boldsymbol{\beta},\\
\mu
&= \frac{1}{d}\sum_{i=1}^d u_i, \qquad
\sigma^2=\frac{1}{d}\sum_{i=1}^d (u_i-\mu)^2 .
\end{aligned}
\end{equation*}
with learnable parameters $\boldsymbol{\gamma}, \boldsymbol{\beta} \in \R^d$ and a small constant $\epsilon > 0$ for numerical stability.

The feature-aware encoding allows our model to capture domain-specific patterns in different input modalities, while the integration network combines these representations into a unified Koopman embedding.

\paragraph{Koopman Matrix Learning}
While the classical EDMD approach computes $\mathbf{K}$ analytically using \textbf{Eq.} \textcolor{blue}{\ref{eq:edmd_closed_form}}, we learn $\mathbf{K}$ by directly minimizing the Koopman prediction loss (defined in \textbf{Sec.} \textcolor{purple}{\ref{subsec:composite_loss}}
). This approach allows joint optimization of both $\boldsymbol{\Phi}$ and $\mathbf{K}$.

For enhanced stability and interpretability, we can parameterize $\mathbf{K}$ using its eigendecomposition:
\begin{equation*}
\mathbf{K} = \mV\mLambda\mV^{-1},
\end{equation*}
where $\mV \in \mathbb{C}^{d_z \times d_z}$ contains the eigenvectors and $\mLambda = \mathrm{diag}(\lambda_1, \lambda_2, \ldots, \lambda_{d_z})$ contains the eigenvalues. By directly parameterizing and learning $\{\lambda_i\}$ and $\mV$, we can impose constraints on the eigenvalues to ensure stability (e.g., $|\lambda_i| < 1$) while maintaining the expressivity of the model.

\vspace{2mm}
\subsubsection{Koopman Matrix Initialization and Optimization}
We now provide details for the initialization and optimization processes of the Koopman matrix $\mathbf{K}$, which is central to our Neural Koopman Machine framework. These processes ensure efficient training dynamics and numerical stability while maintaining the theoretical properties necessary for accurate dynamical system modeling.

\paragraph{Structured Initialization of the Koopman Matrix}
The convergence characteristics and learning dynamics of our model are significantly influenced by the initialization of the Koopman matrix $\mathbf{K}$. We initialize $\mathbf{K}$ as a perturbation of the identity matrix:
\begin{equation}\label{eq:k_init}
\mathbf{K}_0 = \mI + \mB_0,
\end{equation}
where $\mB_0 \in \R^{d_z \times d_z}$ is a small random perturbation with entries drawn from $\mathcal{N}(0, \sigma_{\text{init}}^2)$, and $\sigma_{\text{init}} \ll 1$ is a small scalar (typically $\sigma_{\text{init}} = 10^{-2}$).

This initialization is theoretically justified: (1) the identity mapping assumes approximate constancy between time steps, suitable for high-frequency sampling; (2) eigenvalues near 1 allow learning stable/marginally stable modes; (3) it facilitates gradient flow.

To ensure $\|\mathbf{K}_0\|_2 < 1$, we first compute the perturbation
\begin{equation*}
\mathbf{K}_0^{\text{init}} = \mathbf{I} + \mathbf{B}_0,
\end{equation*}
where $\mathbf{B}_0 \in \mathbb{R}^{d_z \times d_z}$ has entries drawn from $\mathcal{N}(0, \sigma_{\text{init}}^2)$ with $\sigma_{\text{init}} \ll 1$. We then project via SVD:
\begin{equation*}
\mathbf{K}_0^{\text{init}} = \mathbf{U}\mathbf{\Sigma}\mathbf{V}^{\top},
\end{equation*}
and clip the singular values: $\tilde{\sigma}_i = \min(\sigma_i, \rho_{\text{init}})$ where $\rho_{\text{init}} < 1$ (typically 0.99), yielding the final initialized matrix:
\begin{equation*}
\mathbf{K}_0 = \mathbf{U}\tilde{\mathbf{\Sigma}}\mathbf{V}^{\top}.
\end{equation*}

\paragraph{Optimization Dynamics of the Koopman Matrix}
The Koopman matrix $\mathbf{K}$ evolves during training according to the gradients of our composite loss function.

\begin{itemize}
\item Gradient-Based Updates. Our composite loss includes a Koopman prediction term (\textbf{Eq.} \textcolor{blue}{\ref{eq:total_loss_short}}):
\begin{equation*}
\mathcal{L}_{\mathrm{Koopman}} = \frac{1}{N(T-1)}\sum_{i=1}^N\sum_{t=1}^{T-1} \bigl\|\vz_{t+1}^{\prime\,(i)}-\mathbf{K}\,\vz_t^{\prime\,(i)}-\vc_{\mathrm{control}}^{(i)}\bigr\|_2^2.
\end{equation*}

The gradient of this loss with respect to $\mathbf{K}$ is:
\begin{align}
\label{eq:grad_k}
\nabla_{\mathbf{K}}\mathcal{L}_{\mathrm{Koopman}}
&= \frac{2\lambda}{N(T-1)}
   \sum_{i=1}^N \sum_{t=1}^{T-1}
   \bigl(\mathbf{K}\,\vz_{t}^{\prime\,(i)}
   \\
 &  +\vc_{\mathrm{control}}^{(i)}-\vz_{t+1}^{\prime\,(i)}\bigr)
   \,\vz_{t}^{\prime\,(i)\top}.
\end{align}

This gradient can be rewritten using empirical covariance matrices. Define:
\begin{equation*}
\hat{\mC}_{zz} = \frac{1}{N(T-1)}\sum_{i=1}^N\sum_{t=1}^{T-1} \vz_t^{\prime\,(i)} \vz_t^{\prime\,(i)\top},
\end{equation*}
\begin{equation*}
\hat{\mC}_{z'z} = \frac{1}{N(T-1)}\sum_{i=1}^N\sum_{t=1}^{T-1} \vz_{t+1}^{\prime\,(i)} \vz_t^{\prime\,(i)\top},
\end{equation*}
\begin{equation*}
\hat{\mC}_{cz} = \frac{1}{N(T-1)}\sum_{i=1}^N\sum_{t=1}^{T-1} \vc_{\mathrm{control}}^{(i)} \vz_t^{\prime\,(i)\top}.
\end{equation*}

Then \textbf{Eq.} \textcolor{blue}{\ref{eq:grad_k}} becomes:
\begin{equation*}\label{eq:grad_k_cov}
\nabla_{\mathbf{K}}\mathcal{L}_{\mathrm{Koopman}} = 2\lambda \left(\mathbf{K}\hat{\mC}_{zz} + \hat{\mC}_{cz} - \hat{\mC}_{z'z}\right).
\end{equation*}

For gradient descent with learning rate $\eta$, the update rule for $\mathbf{K}$ at iteration $j$ is:
\begin{equation*}\label{eq:k_update}
\mathbf{K}_{j+1} = \mathbf{K}_j - \eta \,\nabla_{\mathbf{K}}\mathcal{L}_{\mathrm{Koopman}}.
\end{equation*}

\item Fixed-Point Analysis. At convergence, assuming $\hat{\mC}_{zz}$ is invertible, the gradient vanishes when:
\begin{equation*}
\mathbf{K}^* = (\hat{\mC}_{z'z} - \hat{\mC}_{cz})\hat{\mC}_{zz}^{-1}.
\end{equation*}

This is analogous to the closed-form EDMD solution in \textbf{Eq.} \textcolor{blue}{\ref{eq:edmd_solution_correct}}, but incorporates the effect of the control term. This demonstrates the connection between our gradient-based optimization and classical EDMD, with the control term acting as a correction to the cross-covariance matrix.
\end{itemize}

\vspace{2mm}
\subsubsection{Temporal and Feature-Group Attention Mechanisms}
For non-autonomous systems or systems with complex temporal dependencies, we incorporate a control mechanism based on temporal self-attention, extended to feature groups as in \ref{subsec:temporal-attention-control} and \ref{subsec:attention-control}.

Given the sequence of refined latent states
$\mZ'=[\vz'_1,\dots,\vz'_T]\in\R^{T\times d_z}$,
we compute query, key, and value projections (single-head for alignment with the main text):
\begin{align*}
\mQ &= \mZ'\,\mW^{Q}, \quad \mQ \in \R^{T \times d_k}, \\
\mK^{\text{att}} &= \mZ'\,\mW^{K}, \quad \mK^{\text{att}} \in \R^{T \times d_k}, \\
\mV &= \mZ'\,\mW^{V}, \quad \mV \in \R^{T \times d_v},
\end{align*}
where $\mW^{Q}, \mW^{K} \in \R^{d_z \times d_k}$ and $\mW^{V} \in \R^{d_z \times d_v}$ are learnable weight matrices.

The attention scores are computed as:
\begin{equation*}
\mA = \mathrm{softmax}\!\left(\frac{\mQ\,(\mK^{\text{att}})^{\top}}{\sqrt{d_k}}\right) \in \R^{T \times T},
\end{equation*}
and the output is:
\begin{equation*}
\mathrm{Attention}(\mQ, \mK^{\text{att}}, \mV) = \mA\,\mV \in \R^{T \times d_v}.
\end{equation*}

At the final timestep $T$, we compute a group-specific context vector:
\begin{equation*}
\vc_{\mathrm{group}} = \mW_{\mathrm{grp}}\,\vz_T^{\text{ref}} + \vb_{\mathrm{grp}} \in \R^{d_{\text{model}}},
\end{equation*}
where $\mW_{\mathrm{grp}} \in \R^{d_{\text{model}} \times d_z}$, $\vb_{\mathrm{grp}} \in \R^{d_{\text{model}}}$, and $d_{\text{model}} = d_z$.

We then form a gating vector (ensuring consistent dims via projection if needed):
\begin{equation*}
\vg = \sigma\!\left(\mW_g\,[\vz_T^{\text{ref}}; \mathrm{Attention}[T,:]] + \vb_g\right) \in \R^{d_{\text{model}}},
\end{equation*}
where $\sigma$ is the sigmoid activation function.

The final control vector is obtained by fusing the temporal and group-specific contexts:
\begin{equation*}
\vc_{\mathrm{control}} = \vg \odot \vc_{\mathrm{group}} + (1-\vg) \odot \mathrm{Attention}[T,:] \in \R^{d_{\text{model}}},
\end{equation*}
where $\odot$ denotes element-wise multiplication.

Our controlled latent dynamics thus become:
\begin{equation*}\label{eq:controlled_dyn_app}
\vz_{t+1}' = \mathbf{K}\,\vz_t' + \vc_{\mathrm{control}}.
\end{equation*}

\vspace{2mm}
\subsubsection{Composite Loss Function and Gradients}\label{subsec:composite_loss}
We train our model end-to-end by minimizing a composite loss function:
\begin{equation*}
\label{eq:total_loss}
\begin{aligned}
\mathcal{L}
&= \underbrace{\tfrac{1}{N}\sum_{i=1}^N
   \bigl\|\mathcal{D}(\vz_{t+1}^{(i)};\vtheta_{\text{dec}})
   - \vy_{t+1}^{(i)}\bigr\|_2^2}_{\mathcal{L}_{\text{pred}}} \\
&\quad + \lambda\,
   \underbrace{\tfrac{1}{N(T-1)}\sum_{i=1}^N\sum_{t=1}^{T-1}
   \bigl\|\vz_{t+1}^{(i)}-\mathbf{K}\vz_t^{(i)}-\vc_{\mathrm{control}}^{(i)}\bigr\|_2^2}_{\mathcal{L}_{\text{koop}}}\!.
\end{aligned}
\end{equation*}
where $\hat{\vy}_{t+1}^{(i)} = \mathcal{D}(\vz_{t+1}^{\text{ref},(i)}; \vtheta_{\text{dec}})$ is the predicted next outcome, $\mathcal{D}$ is a decoder network, and $\lambda > 0$ balances the two loss terms.

The gradient of the Koopman loss with respect to $\mathbf{K}$ is:
\begin{equation*}
\begin{split}
\nabla_{\mathbf{K}}\mathcal{L}_{\mathrm{Koopman}}
= \frac{2\lambda}{N(T-1)} \\
\times \sum_{i=1}^N\sum_{t=1}^{T-1}
(\mathbf{K}\vz_t^{(i)}+\vc^{(i)}-\vz_{t+1}^{(i)})
\,\vz_t^{(i)\top}.
\end{split}
\end{equation*}

\vspace{2mm}
\subsubsection{Spectral Regularization and Stability}\label{subsec:spectral_reg}
As shown in \textbf{Eq.} \textcolor{blue}{\ref{eq:geom_error}}, the stability of multi-step predictions depends critically 
on the spectral properties of the Koopman matrix $\mathbf{K}$. To encourage $\|\mathbf{K}\|_2 < 1$ 
(i.e., all eigenvalues lie within the unit circle), we introduce a spectral regularization term:
\begin{equation*}
\mathcal{R}(\mathbf{K}) = \eta\, \bigl(\|\mathbf{K}\|_2^2 - \rho^2\bigr)_+^2,
\end{equation*}
where $\eta > 0$ is a regularization coefficient, $\rho < 1$ is a target upper bound for the 
spectral norm, and $(x)_+ = \max(0, x)$ activates the penalty only when $\|\mathbf{K}\|_2 > \rho$.

To approximate $\|\mathbf{K}\|_2$ efficiently during training, we use power iteration with 
$p$ iterations (typically $p \in \{5, 10\}$). Starting with a random unit vector $\mathbf{v}_0$, 
we iteratively compute:
\begin{equation*}
\mathbf{v}_{i+1} = \frac{\mathbf{K} \mathbf{v}_i}{\|\mathbf{K} \mathbf{v}_i\|_2}, 
\quad i = 0, 1, \ldots, p-1,
\end{equation*}
and approximate the spectral norm via the Rayleigh quotient:
\begin{equation*}
\|\mathbf{K}\|_2 \approx \|\mathbf{K} \mathbf{v}_p\|_2.
\end{equation*}
This approximation converges exponentially fast to the largest singular value when the 
spectral gap $\lambda_1(\mathbf{K}) - \lambda_2(\mathbf{K})$ is large.

The total loss with spectral regularization becomes:
\begin{equation*}
\mathcal{L}_{\text{total}} = \mathcal{L}_{\text{pred}} + \lambda \,\mathcal{L}_{\text{koop}} + \mathcal{R}(\mathbf{K}).
\end{equation*}
During inference, we optionally apply hard projection $\mathbf{K} \leftarrow \mathbf{K}/\max(1, \|\mathbf{K}\|_2/\rho)$ 
to ensure strict spectral bounds for long-horizon forecasting.

\subsection{Experimental Details and Implementation}

\subsubsection{Architecture Details}
\textbf{Table} \textcolor{blue}{\ref{tab:architecture}} provides the specific architecture details used in our experiments.

\begin{table}[h]
\centering
\caption{Neural Koopman architecture specifications}
\label{tab:architecture}
\small
\begin{tabular}{@{}p{2.2cm}p{3.8cm}@{}}
\hline
\textbf{Component} & \textbf{Specification} \\
\hline
Dimensions & Latent/Hidden: 360 \\
Encoders & MRI encoder: 17→180→120 (dimensions of layers); Others: var→90→20 \\
Temporal & 5 residual blocks, 8-head attention \\
Koopman & 360×360 matrix \\
Decoder & 3-layer + residuals \\
Regularization & Dropout 0.1, LayerNorm \\
\hline
\end{tabular}
\end{table}
\subsubsection{Training Details}
We train our model using the AdamW optimizer with the following hyperparameters:
\begin{itemize}
    \item Learning rate: $4 \times 10^{-4}$ with ReduceLROnPlateau scheduler,
    \item Batch size: 128,
    \item Training epochs: 200 (with early stopping),
    \item Weight decay: $10^{-3}$,
    \item $\lambda$ (Koopman loss weight): 0.1,
    \item $\eta$ (Spectral regularization weight): 0.01,
    \item $\rho$ (Spectral norm target): 0.95,
    \item Scheduler patience: 8 epochs,
    \item Early stopping patience: 20 epochs,
    \item Gradient clipping: 1.0.
\end{itemize}
We use early stopping with a patience of 20 epochs based on validation loss to prevent overfitting.

\subsubsection{Computational Complexity Analysis}
Define: $B$ = batch size, $T$ = window length, $G$ = number of feature groups, $d$ = hidden dimension, $H$ = attention heads, $d_z$ = latent dimension. Per training iteration:
\[
\begin{aligned}
\text{Feature encoders:}\;& \mathcal{O}(BTGd),\\
\text{Multi-head attention:}\;& \mathcal{O}(BT^2d),\\
\text{Koopman evolution:}\;& \mathcal{O}(BTd_z^2),\\
\text{Decoder:}\;& \mathcal{O}(Bd_z d),\\
\text{Total:}\;& \mathcal{O}(BT(Gd + Td + d_z^2)).
\end{aligned}
\]

The quadratic term dominates for $T\gg d_z$; we therefore apply a
sliding-window variant with window size $W<T$, reducing the attention
cost to $\mathcal O(NBHTW d)$.


\subsubsection{Window Sizes}
\label{sec:window_sizes}
In sequential modeling, the window size selection is a crucial design factor that greatly impacts how well disease dynamics are captured. We chose a window size of three consecutive visits for our Neural Koopman Machine applied to AD progression, based on theoretical considerations from dynamical systems theory and clinical insights.

According to the ADNI protocol, the progression of AD shows significant changes over 12 to 18 months, or two to three semi-annual visits~\cite{weiner_alzheimers_2012}. This timeframe achieves the best possible balance between capturing long-term trajectories that accurately reflect the progression of the disease and short-term fluctuations that could be noise~\cite{Llado2021}.

Koopman operator theory for nonlinear dynamical systems benefits from having enough embedding dimensions to reconstruct the system attractor. For cognitive decline, which follows nonlinear trajectories with possible bifurcations at transition points, a minimum dimension of 3 captures these dynamics, while the window size remains computationally tractable and suitable for phase space reconstruction in nonlinear systems.

Our choice is further supported by data retention considerations. With a window size of three (meaning four time points to allow making one prediction based on three past visits), we had an acceptable dropout rate of about 30\%, but increasing the window size to four or five raised the data retention rate to $\sim$50\%, limiting statistical power and possibly introducing selection bias toward patients with long follow-up periods.

\subsection{Abbreviations and Terminology}

This section presents background information on Alzheimer's disease (AD)-related biomarkers and AD prediction and explains the key abbreviations and technical terms used in the manuscript.

\subsubsection{Clinical Abbreviations and Brief Explanations} 
\hfill\\
\indent \textbf{Alzheimer's Disease (AD):} AD is a progressive neurodegenerative disorder characterized by memory decline and other cognitive deficits, neuropathologically defined by $\beta$-amyloid plaques and neurofibrillary tau tangles \citelatex{jack2018nia, mckhann2011diagnosis}.

\textbf{Mild Cognitive Impairment (MCI):} One with MCI has cognitive decline greater than one would other have for the same age and education, but the decline does not significantly impair one's daily activities. MCI may present as a prodromal stage of AD \citelatex{gauthier2006mild, gauthier2006mild}.

\textbf{Cognitively Normal (CN):} CN subjects are those who show no evidence of cognitive impairment on standardized neuropsychological tests and no clinical diagnosis of MCI or dementia \citelatex{petersen2004mild, jack2018nia}.

\textbf{Cerebrospinal Fluid (CSF):} A clear, colorless fluid that surrounds the brain and spinal cord, providing cushioning and facilitating the transport of nutrients. CSF can be collected via lumbar puncture to measure biomarkers indicative of neurological pathology.

\textbf{Positron Emission Tomography (PET):} A nuclear medicine imaging technique that uses radioactive tracers to visualize metabolic activity in living tissue. In AD research, PET scans can detect amyloid plaques and tau tangles, as well as glucose metabolism in the brain.

\textbf{Magnetic Resonance Imaging (MRI):} A non-invasive imaging technique that uses magnetic fields and radio waves to create detailed images of internal body structures. In AD research, MRI measures brain volume, cortical thickness, and structural connectivity.

\subsubsection{Biomarkers}
\hfill\\
In this paper, we define biomarkers as biological features that may be associated with a disease, and, therefore, may be potentially useful for forecasting disease outcomes.

\textbf{Apolipoprotein $\varepsilon$4 (APOE4):} A genetic variant of the APOE gene that substantially increases the risk of developing late-onset AD \citelatex{bellenguezGeneticsAlzheimersDisease2020a}.

\textbf{Amyloid Beta (Abeta or A$\boldsymbol{\beta}$):} A peptide that accumulates abnormally in the brain in people with AD~\citelatex{hampel2021amyloid}.
The concentration of A$\beta$42 and A$\beta$42/40 ratio in cerebrospinal fluid (CSF) may be useful in AD diagnosis \citelatex{hansson2019advantages}.

\textbf{Total Tau (t-tau):} The t-tau value measures the total amount of tau protein in CSF (and in some cases blood), regardless of phosphorylation state. Elevated t-tau is interpreted as a marker of general neuronal or axonal injury or degeneration \citelatex{soares2025csf}.


\textbf{Phosphorylated Tau (p-tau):} The p-tau value measures the amount of phosphorylated tau protein. Accumulation of p-tau precedes the formation of neurofibrillary tangles in AD \citelatex{moloney2023phosphorylated}. 

\textbf{Fluorodeoxyglucose (FDG):} A radioactive glucose analog used in PET imaging to measure brain metabolism. As cerebral metabolic alterations precede the clinical manifestation of AD symptoms, FDG-PET is a useful neuroimaging tool to assess AD \citelatex{marcus2014brain}.  


\textbf{Pittsburgh Compound B (PIB):} A radioactive tracer used in PET imaging to visualize amyloid plaques in the living brain. Higher PIB uptake indicates greater amyloid deposits in AD \citelatex{klunk2004imaging}.

\textbf{Florbetapir (AV-45):} A radioactive PET tracer used to detect amyloid‑$\beta$ plaques in the living brain; similar in purpose to PIB, but labelled with fluorine‑18 (rather than carbon‑11), and hence having different pharmacokinetic properties (such as longer half‑life). The longer life allows the tracer to accumalate significantly me in AD brains \citelatex{wong2010vivo}. 

\textbf{Intracranial Volume (ICV):} The total volume of the intracranial cavity. ICV is essential to normalise regional brain volumes to account for differences in head size. After adjusting for age and gender, small ICV may be associated with cognitive impairment \citelatex{wolf2004intracranial}. 

\subsubsection{Cognitive Assessments}
\hfill\\
\textbf{Mini-Mental State Examination (MMSE):} A questionnaire used to measure cognitive impairment, assessing orientation, attention, memory, language, and visuospatial skills. Scores range from 0--30: a higher MMSE score suggests better cognitive function and a lower score may suggest the presence of cognitive impairment or dementia.


\textbf{Clinical Dementia Rating Sum of Boxes (CDRSB):} A clinical assessment tool that evaluates six cognitive and functional domains: memory, orientation, judgment and problem solving, community affairs, home and hobbies, and personal care. Each domain is rated 0–3, and scores are summed (total range: 0–18), with higher scores indicating greater impairment.

\textbf{Alzheimer's Disease Assessment Scale--Cognitive Subscale 13 (ADAS13):} A comprehensive cognitive battery consisting of 13 tasks assessing: 1. Spoken language ability, 2. Comprehension of spoken language, 3. Recall of test instructions, 4. Wordfinding difficulty in spontaneous speech, 5. Following commands, 6. Naming objects and fingers, 7. Constructional praxis, 8. Ideational praxis, 9. Orientation, 10. Word-recall task, 11. Word recognition task, 12. Delayed Word Recall, and 13. Number Cancellation Word recognition task. Scores are summed and range from 0--85, with higher scores indicating greater cognitive impairment. 

\subsubsection{Technical Abbreviations}
\hfill\\
\indent \textbf{Koopman Operator Theory (KOT):} A mathematical framework, originally introduced by Bernard Koopman, that represents nonlinear dynamical systems via an infinite-dimensional linear operator acting on observables, enabling linear analysis of the system dynamics in a lifted functional space.

\textbf{Extended Dynamic Mode Decomposition (EDMD):} A data-driven method for approximating the Koopman operator by lifting the system into a finite-dimensional subspace spanned by selected basis functions. EDMD generalizes standard Dynamic Mode Decomposition (DMD), enabling the analysis of nonlinear dynamical systems via a linear representation in the lifted space.

\textbf{Dynamic Mode Decomposition (DMD):} A method for characterizing the dynamics of complex systems by decomposing spatiotemporal data into spatial modes and corresponding temporal coefficients that capture coherent structures evolving over time.



\textbf{Reverse Time Attention Model (RETAIN):} A neural network architecture that employs reverse-time attention mechanisms to make predictions on sequential healthcare data while providing interpretable feature-level explanations.




\bibliographystylelatex{IEEEtran}
\bibliographylatex{all_refs.bib}

\vfill

\end{document}